\documentclass{article} % For LaTeX2e
\usepackage{times}
\PassOptionsToPackage{numbers,sort&compress}{natbib}
\usepackage[preprint]{neurips_2026}
\usepackage{amsmath,amsfonts,bm,amsthm}

\def\1{\bm{1}}

\def\vzero{{\bm{0}}}
\def\vone{{\bm{1}}}

\def\va{{\bm{a}}}
\def\vb{{\bm{b}}}
\def\vc{{\bm{c}}}

\def\vo{{\bm{o}}}

\def\vu{{\bm{u}}}
\def\vv{{\bm{v}}}

\def\vx{{\bm{x}}}

\def\vepsilon{{\bm{\epsilon}}}

\DeclareMathAlphabet{\mathsfit}{\encodingdefault}{\sfdefault}{m}{sl}
\SetMathAlphabet{\mathsfit}{bold}{\encodingdefault}{\sfdefault}{bx}{n}

\newcommand{\E}{\mathbb{E}}

\DeclareMathOperator*{\argmin}{arg\,min}

\usepackage{hyperref}
\usepackage{url}
\usepackage{algorithm}
\usepackage{algpseudocode}
\algrenewcommand\algorithmiccomment[1]{\hfill\small\textcolor{gray}{\textit{// #1}}}
\usepackage{booktabs}
\usepackage{multirow}
\usepackage{graphicx}
\usepackage{colortbl,xcolor}
\usepackage{mathtools}
\usepackage{caption}
\usepackage{wrapfig}
\usepackage{amssymb}

\newcommand{\ourmethod}{ReCAST}
\newcommand{\oursshort}{Ours}

\title{ReCAST: Reward Credit Assignment across Timesteps for Online Diffusion Reinforcement}

\author{
Yihang Chen$^{1,}$\thanks{Equal contribution.} \quad
Yuanhao Ban$^{1,2,}$\footnotemark[1] \quad
Kuei-Chun Kao$^{1,2}$ \quad
Cho-Jui Hsieh$^{1,2}$ \\
$^{1}$University of California, Los Angeles\quad 
$^{2}$Arena AI \\
\texttt{\{yhangchen,banyh2000\}@cs.ucla.edu}
}
\theoremstyle{plain}
\newtheorem{theorem}{Theorem}[section]

\theoremstyle{definition}

\theoremstyle{remark}

\begin{document}

% modest float/display spacing (keeps the main body within the 9-page limit)
\setlength{\textfloatsep}{8pt plus 2pt minus 2pt}
\setlength{\floatsep}{8pt plus 2pt minus 2pt}
\setlength{\intextsep}{8pt plus 2pt minus 2pt}
\captionsetup{skip=4pt,font=small}
\setlength{\abovedisplayskip}{4pt plus 1pt minus 1pt}
\setlength{\belowdisplayskip}{4pt plus 1pt minus 1pt}
\setlength{\abovedisplayshortskip}{4pt plus 2pt minus 2pt}
\setlength{\belowdisplayshortskip}{4pt plus 2pt minus 2pt}

\maketitle
\begin{abstract}

Training diffusion models with multiple rewards requires distinguishing \emph{user preference} from \emph{reward informativeness}. User preference determines how much each reward should contribute to the overall objective; reward informativeness determines when its feedback is useful during denoising. Some rewards can meaningfully evaluate a sample as soon as global structure emerges, but others become informative only when the sample is nearly clean. To address both questions jointly, we propose ReCAST (\textbf{Re}ward \textbf{C}redit \textbf{AS}signment across \textbf{T}imesteps), the first method, to our knowledge, for per-reward, timestep-dependent credit assignment in diffusion reward fine-tuning. ReCAST separates user preferences from temporal allocation through a reward-by-timestep weight matrix $\bm{W}$, whose row sums match the user-specified reward budgets $\bm{\lambda}$, while its column sums are equal, assigning the same total weight to each denoising step. Under these marginal constraints, ReCAST allocates weight according to each reward's informativeness, quantified by its R\'enyi discriminability gain at each step. These gains telescope to the total discriminability between the reward-induced positive policy and the current policy, providing a basis for temporal credit assignment. 
We evaluate ReCAST by training SD3.5-Medium under two distinct four-reward settings, each across five reward budgets $\bm{\lambda}$. ReCAST improves the training rewards in one setting and matches them in the other, improves every held-out judge in both, and is preferred by an independent LLM-as-a-Judge. Overall, these results show that ReCAST achieves improvements that generalize beyond the training rewards and support its core principle: assigning each reward greater weight at the denoising timesteps where its feedback is most informative.
\end{abstract}

\section{Introduction}
\label{sec:intro}

Reward-based post-training has become a standard way to align diffusion and flow models with human intent, via policy gradient over the denoising chain~\citep{black2023training,fan2023dpok}, backpropagation through a differentiable reward~\citep{clark2023directly,prabhudesai2023aligning}, preference optimization~\citep{wallace2024diffusion,zhu2025dspo}, group-relative objectives~\citep{liu2025flow,xue2025dancegrpo}, or negative-aware fine-tuning on the forward process~\citep{diffusionnft}. In practice, a single reward is rarely sufficient: standard training recipes mix prompt alignment~\citep{hessel2021clipscore}, learned human preference~\citep{wu2023hpsv2,kirstain2023pick,xu2023imagereward}, and rule-based correctness~\citep{ghosh2023geneval} by linear scalarization $r(\vx_0,\vc) = \sum_i \lambda_i\, r_i(\vx_0,\vc)$ with hand-set $\lambda_i\ge 0$, $\sum_i \lambda_i=1$. Over-optimizing any single proxy can easily lead to reward hacking~\citep{gao2023scaling,skalse2022defining}. 

However, static scalarization makes an implicit but imperfect assumption: once a reward is assigned an overall budget $\lambda_i$, its relative influence remains the same across all diffusion timesteps. Instead, we want a reward to have greater influence when it is more informative. Multi-reward alignment therefore involves two allocation problems: how much budget each reward receives, and when that budget should be spent. More specifically, given a user-specified reward budget $\bm{\lambda}$, can we allocate each reward's weights across timesteps more efficiently?

Our key intuition is that each reward measures a different visual property, and those properties do not become visible at the same point in the denoising trajectory. A reward for prompt alignment or rule-based correctness~\citep{hessel2021clipscore,ghosh2023geneval} can already grade a sample once coarse layout and composition exist, while a reward for aesthetic or fine local detail cannot say much until the sample is nearly clean. Diffusion timesteps govern different aspects of generation by construction, with high-noise steps fixing global composition and low-noise steps refining local detail~\citep{choi2022perception,hang2023efficient}. A static, timestep-independent $\lambda_i$ ignores this and forces every reward to contribute uniformly across $t$ regardless of whether it can discriminate there. Although timestep-dependent treatment is well established for likelihood training~\citep{kingma2021variational,karras2022elucidating} and reward fine-tuning~\citep{liang2024step,li2025mixgrpo,tempflow2025}, these methods apply a single reward-agnostic schedule, not the per-reward one we study.

Credit assignment is widely analyzed for language-model reasoning~\citep{uesato2022solving,lightman2024lets,kazemnejad2024vineppo} but largely unexplored in the diffusion denoising chain. We propose ReCAST (\textbf{Re}ward \textbf{C}redit \textbf{AS}signment across \textbf{T}imesteps), the first per-reward timestep-dependent credit assignment method for diffusion reward fine-tuning, to our knowledge. More specifically, for rewards $r_1,\dots,r_m$ and $T$ steps, we replace the static convex weighting by an \emph{automatically} designed {weight matrix} $\bm{W}\in\mathbb R_{\geq 0}^{m\times T}$, whose entry $W_{i,t}$ is the weight reward $i$ carries at step $t$, and define the reward at timestep $t$ as
\begin{equation}
\label{eq:two-level-master}
\boxed{\;r_t(\vx_0,\vc) \;=\; \sum_{i=1}^m T\cdot W_{i,t}\;r_i(\vx_0,\vc),\qquad
\underbrace{\textstyle\sum_{t} W_{i,t} = \lambda_i}_{\substack{\text{inter-reward row budget}}},\qquad
\underbrace{\textstyle\sum_{i} W_{i,t} = 1/T}_{\substack{\text{per-step column budget}}}.\;}
\end{equation}
  The row marginal $\lambda_i\geq 0$ ($\sum_i\lambda_i=1$) is the inter-reward budget set by the user, and states \emph{how much} reward $i$ counts. The column marginal enforces that every
  step receives the same total weight $1/T$, so no step is starved or over-optimized. Together the two marginals fix the totals but not the shape: how row $i$ spreads its budget $\lambda_i$
  across $t$ is precisely \emph{when} reward $i$ counts, and that is what we estimate rather than hand-set. We first measure each reward's own per-step gain in R\'enyi discriminability along
  the trajectory (Sec.~\ref{sec:Dt}--\ref{sec:estimation}), then build a kernel from these gains and project it onto the feasible set of Eq.~\eqref{eq:two-level-master} to give the desired $\bm{W}^\star$ (Sec.~\ref{sec:sinkhorn}).
\section{Related work}
\label{sec:related}

\paragraph{RL in diffusion and flow models.}
Reward fine-tuning is central to current text-to-image models. DanceGRPO~\citep{xue2025dancegrpo} extends group-relative optimization to both diffusion and rectified-flow image and video generators within a single framework; Flow-GRPO~\citep{liu2025flow} converts the flow-matching ODE into an SDE so group sampling and its advantage estimator stay well defined; DiffusionNFT~\citep{diffusionnft} regresses implicit positive/negative velocity fields with a forward-only loss. Our work is built upon DiffusionNFT due to its performance and efficiency.

\paragraph{Combining multiple rewards.}
How multiple reward signals are combined has received wide attention in areas not limited to text-to-image tasks. SafeRLHF~\citep{dai2023saferlhf} decouples helpfulness and harmlessness into separate reward/cost models balanced via a Lagrange multiplier rather than a fixed scalar; Personalized Soups~\citep{jang2023personalizedsoups} instead trains one policy per preference dimension and merges them post-hoc, avoiding scalarization altogether, and Rewarded Soups~\citep{rame2023rewarded} takes the same merging route for reward proxies to reach Pareto trade-offs by interpolating weights. ALaRM~\citep{lai2024alarm} organizes rewards hierarchically rather than flattening them into one sum, and large reasoning models combine a rule-based outcome reward, a length penalty, and a language consistency reward within a single objective. GDPO~\citep{gdpo2026} shows that summing rewards before GRPO's group normalization collapses into near-identical advantages, and fixes this by normalizing per reward before recombining. None of this work varies a reward's weight across denoising timesteps: our contribution is orthogonal to how the mixture itself is formed.

\paragraph{Timestep weighting in diffusion training.}
Existing works have adopted unequal treatment of denoising steps in diffusion training, such as ELBO-consistent weightings~\citep{song2021maximum,kingma2021variational}, perception-prioritized and Min-SNR schedules~\citep{choi2022perception,hang2023efficient}, EDM weighting~\citep{karras2022elucidating}, and importance sampling of $t$~\citep{nichol2021improved}. On the reward side, step-by-step preference optimization~\citep{liang2024step} supervises every denoising step separately through a step-aware preference model, while MixGRPO's sliding ODE/SDE window~\citep{li2025mixgrpo} and TempFlow-GRPO's noise-aware weighting~\citep{tempflow2025} concentrate the reward signal on high-noise steps. All show that \emph{when} a signal is applied matters, but each shares one heuristic curve across rewards; ours is per-reward, from its own discriminability curve.

\section{Method}
\label{sec:renyi-reweight}

This section builds the per-step gain in reward discriminability, then turns those gains into the weight matrix $\bm{W}$. In Sec.~\ref{sec:Dt}--\ref{sec:estimation} we look at one reward, asking only how the per-step gain is shaped across $t$; Sec.~\ref{sec:sinkhorn} then puts the per-reward shapes into a matrix that satisfies both marginals of Eq.~\eqref{eq:two-level-master}. 

We work in the rectified-flow setup and with the DiffusionNFT objective, and put the background in App.~\ref{sec:background}. Throughout this section we assume $r_i\geq0$.

\subsection{Setup and the density ratio}
For each reward $r_i$, $i\in\{1,\dots,m\}$, let $\pi_{i}^+(\vx_0\mid\vc)$ be the positive policy induced by reward $i$ on the old policy $\pi^{\rm old}$,
\begin{equation}
\label{eq:pi+def}
    \pi_{i}^+(\vx_0\mid\vc) := \frac{r_i(\vx_0,\vc)}{\E_{\pi^{\mathrm{old}}(\vx_0\mid\vc)}[r_i(\vx_0,\vc)]}\,\pi^{\mathrm{old}}(\vx_0\mid\vc)
\end{equation}
and let $\pi_{i,t}^+(\vx_t\mid\vc)$ and $\pi_t^{\mathrm{old}}(\vx_t\mid\vc)$ be the marginals at diffusion timestep $t$. Define the density ratio of $r_i$, $\rho_{i,t}(\vx_t) := {\pi_{i,t}^+(\vx_t\mid\vc)}/{\pi_t^{\mathrm{old}}(\vx_t\mid\vc)} 
$. 
Marginalizing the forward kernel and applying Bayes' rule, $\pi^{\mathrm{old}}(\vx_0\mid\vx_t,\vc) = p(\vx_t\mid\vx_0)\pi^{\mathrm{old}}(\vx_0\mid\vc)\,/\,\pi_t^{\mathrm{old}}(\vx_t\mid\vc)$, gives the ratio at timestep $t$ as a posterior expected reward normalized by the marginal one:
\begin{equation}
\label{eq:alpha-reward}
{\;\rho_{i,t}(\vx_t) \;:=\; \frac{\pi_{i,t}^+(\vx_t\mid\vc)}{\pi_t^{\mathrm{old}}(\vx_t\mid\vc)}  \;=\;  \frac{\int p(\vx_t\mid\vx_0)\,\pi_{i}^+(\vx_0\mid\vc)\,{\rm d}\vx_0}{\pi_t^{\mathrm{old}}(\vx_t\mid\vc)} \;=\; \frac{\E_{\pi^{\mathrm{old}}(\vx_0\mid\vx_t,\vc)}\!\big[r_i(\vx_0,\vc)\big]}{\E_{\pi^{\mathrm{old}}(\vx_0\mid\vc)}\!\big[r_i(\vx_0,\vc)\big]}.\;}
\end{equation}
Everything on the right-hand side is a reward evaluation under the current policy, so the ratio is theoretically computable without ever directly sampling from $\pi_{i}^+$. For simplicity, we write $\mu_{i,t}(\vx_t) := \E_{\pi^{\mathrm{old}}(\vx_0\mid\vx_t,\vc)}[r_i]$ for the conditional reward and $Z_i := \E_{\pi^{\mathrm{old}}(\vx_0\mid\vc)}[r_i]$ for its marginal, so $\rho_{i,t}(\vx_t) = \mu_{i,t}(\vx_t)/Z_i$. 

\subsection{Reward discriminability}
\label{sec:Dt}
We first ask one question: how distinguishable is the positive policy from the current one at a given timestep? We measure this by the R\'enyi divergence~\citep{renyi1961measures} of order $\alpha>1$. The \emph{cumulative reward discriminability} at timestep $t$ for reward $i$ is therefore
\begin{equation}
\label{eq:Dt}
D_{i,t} \;:=\; D_\alpha\!\big(\pi_{i,t}^+ \,\big\|\, \pi_t^{\mathrm{old}}\big) \;=\; \frac{1}{\alpha-1}\log\E_{\pi_{i,t}^+}\!\Big[\rho_{i,t}(\vx_t)^{\alpha-1}\Big].
\end{equation}
It vanishes exactly when the two marginals agree, grows as reward $i$ separates them, and is nondecreasing in $\alpha$. Its $\alpha\to1^+$ limit is the KL divergence between the same two marginals,
\begin{equation}
\label{eq:kl-limit}
\lim_{\alpha\to1^+} D_{i,t} \;=\; \mathrm{KL}\big(\pi_{i,t}^+\,\big\|\,\pi_t^{\mathrm{old}}\big) \;=\; \E_{\pi_{i,t}^+}\!\big[\log\rho_{i,t}(\vx_t)\big].
\end{equation}
Sec.~\ref{sec:estimation} shows why we choose R\'enyi instead of the KL divergence.

At pure noise ($t=T$) the forward map has $\sigma_T=1$ and discards $\vx_0$ entirely, so $\pi_{i,T}^+ = \pi_T^{\mathrm{old}} = \mathcal N(0,I)$ and $D_{i,T} = 0$ for every reward. At clean data the divergence reaches its maximum $D_{i,0} = D_\alpha(\pi_{i,0}^+\,\|\,\pi_0^{\mathrm{old}})$, the full data-level discriminability reward $i$ induces. 

\paragraph{Discriminability gain.}
Cumulative discriminability $D_{i,t}$ contains all signal accumulated from $T$ to $t$, so weighting by it would repeatedly count earlier gains. Therefore, we use the per-step \emph{gain}
\begin{equation}
\label{eq:gain}
{\Delta D_{i,t} := D_{i,t-1}-D_{i,t},}
\end{equation}
which isolates the new discriminability contributed by reward $i$ at transition $t\to t-1$. The gains telescope along the denoising trajectory:
\begin{equation}
\label{eq:telescope}
{\sum_{t=1}^T \Delta D_{i,t} = D_{i,0}-D_{i,T} = D_{i,0}
= D_\alpha\!\big(\pi_{i,0}^+\,\big\|\,\pi_0^{\mathrm{old}}\big).}
\end{equation}
Thus, $\Delta D_{i,t}$ decomposes the reward's total data-level discriminability across timesteps without double-counting: rewards sensitive to global structure concentrate at large $t$, while those sensitive to fine details concentrate at small $t$.

One basic property is that $\Delta D_{i,t}\geq 0$, so at each denoising step ($t\to t-1$) the discriminability does not decrease. This ensures the soundness of our definition of $\Delta D_{i,t}$.  
We prove a general theorem.

\begin{theorem}[R\'enyi dissipation under shared diffusion]
\label{thm:debruijn}
Let $\pi_t^{a}$ and $\pi_t^{b}$ be positive, sufficiently smooth probability densities on $\mathbb R^d$ that both evolve under the same forward Fokker--Planck equation
\begin{equation}
\label{eq:fp}
\partial_t \pi_t \;=\; -\nabla\!\cdot\!\big(\vu(\vx,t)\,\pi_t\big) \;+\; \tfrac{g(t)^2}{2}\,\Delta\pi_t,
\end{equation}
for a shared drift $\vu$ and spatially constant diffusion coefficient $g(t)$. Assume the displayed moments and derivatives are integrable and that the boundary terms in the integrations by parts vanish. Let $\rho := \pi^a/\pi^b$ denote the density ratio, and define the \emph{R\'enyi-tilted distribution} of order $\alpha$:
\begin{equation}
\label{eq:renyi-tilted}
\pi_t^{(\alpha)}(\vx) \;:=\; \frac{\rho(\vx)^\alpha}{\E_{\pi_t^b}[\rho^\alpha]}\,\pi_t^b(\vx) \;\propto\; \big(\pi_t^a\big)^\alpha\,\big(\pi_t^b\big)^{1-\alpha}.
\end{equation}
This geometric family recovers $\pi_t^b$ at $\alpha=0$ and $\pi_t^a$ at $\alpha=1$; in the regime used here, $\alpha>1$, it extrapolates beyond $\pi_t^a$. Then for every $\alpha>0$ with $\alpha\neq1$,
\begin{equation}
\label{eq:debruijn}
{\;\frac{{\rm d}}{{\rm d}t}\,D_\alpha\!\big(\pi_t^a \,\|\, \pi_t^b\big) \;=\; -\,\frac{g(t)^2}{2}\cdot\alpha\cdot\E_{\pi_t^{(\alpha)}}\!\big[\|\nabla\log\frac{\pi^a}{\pi^b}\|^2\big],\;}
\end{equation}
where $\E_{\pi_t^{(\alpha)}}[\|\nabla\log\frac{\pi^a}{\pi^b}\|^2]$ is the Fisher information of the log-density ratio measured under the tilted distribution $\pi_t^{(\alpha)}$. 
\end{theorem}
The proof of the theorem is deferred to App.~\ref{app:proof-debruijn}. Both $\pi_{i,t}^+$ and $\pi_t^{\mathrm{old}}$ evolve according to Eq.~\eqref{eq:fp} with the same coefficients because the diffusion kernel $p(\vx_t\mid\vx_0)$ is independent of the policy. Away from the singular endpoint $\sigma_t=1$, the forward process $\vx_t=(1-\sigma_t)\vx_0+\sigma_t\vepsilon$
is the transition law of a linear SDE, and therefore satisfies Eq.~\eqref{eq:fp} with $\vu(\vx,t)=-\frac{\dot\sigma_t}{1-\sigma_t}\vx,
\qquad
g(t)^2=\frac{2\sigma_t\dot\sigma_t}{1-\sigma_t}$.
Since Eq.~\eqref{eq:fp} is linear in $\pi_t$, and the reward tilt $r_i$ acts only on the initial state $\vx_0$, the tilt changes only the initial distribution rather than the evolution coefficients. Hence, $\pi_{i,t}^+$ and $\pi_t^{\mathrm{old}}$ follow the same Fokker--Planck dynamics from different initial conditions. Theorem~\ref{thm:debruijn} therefore applies with $\pi^a=\pi_{i,t}^+$, $\pi^b=\pi_t^{\mathrm{old}}$, and $\rho=\rho_{i,t}$, yielding $\frac{{\rm d}}{{\rm d}t}D_\alpha\!\big(\pi_{i,t}^+\,\|\,\pi_t^{\mathrm{old}}\big)\leq 0$.

\subsection{Estimation details}
\label{sec:estimation}
\paragraph{Why we use $\alpha>1$ rather than the KL limit.}
The exact KL in Eq.~\eqref{eq:kl-limit} is well defined under the usual absolute-continuity and integrability conditions: states with $\rho_{i,t}=\pi_{i,t}^+/\pi^{\rm old}_t=0$ have zero probability under $\pi_{i,t}^+$ and therefore do not cause the expectation to diverge. However, empirically in finite-sample rollouts, the estimated conditional reward $\hat{\E}_{\pi^{\mathrm{old}}(\vx_0\mid\vx_t,\vc)}\big[r_i(\vx_0,\vc)\big]$, and hence $\hat\rho_{i,t}$, can be exactly zero, in which case averaging $\log\hat\rho_{i,t}$ introduces $-\infty$ values. For example, consider a noisy state $\vx_t$ that has already committed to a two-object layout when the prompt requires only one. Such an $\vx_t$ will likely be denoised under $\pi^{\rm old}$ into samples $\vx_0$ that keep that structure, so $r_i(\vx_0,\vc)=0$ for multiple $\vx_0$. Simply swapping the two marginals does not resolve this issue: $\mathrm{KL}\big(\pi_t^{\rm old}\,\big\|\,\pi_{i,t}^+\big) = -\E_{\pi_t^{\rm old}}\big[\log\rho_{i,t}(\vx_t)\big]$ may also diverge for the same reason.

This motivates using the R\'enyi divergence with $\alpha>1$. In this case, the logarithm in Eq.~\eqref{eq:Dt} is applied only after the empirical moment $\hat{\E}_{\pi_{i,t}^+}\big[\hat{\rho}_{i,t}(\vx_t)^{\alpha-1}\big]$. Consequently, a zero conditional reward estimate, corresponding to a state $\vx_t$ for which $\hat{\rho}_{i,t}(\vx_t)=0$, contributes zero to the empirical moment rather than producing an infinite logarithmic term. The estimate therefore remains finite whenever the empirical moment itself is positive, which is usually satisfied in practical estimation. At the same time, $D_\alpha$ retains a useful relation to the original objective: because R\'enyi divergence is nondecreasing in its order, $D_\alpha$ upper-bounds the KL for $\alpha>1$. App.~\ref{app:alpha-sweep} examines how the resulting discriminability curve varies with $\alpha$. We use $\alpha=2$ in all training experiments for its simplicity: the exponent in Eq.~\eqref{eq:Dt} is then $\alpha-1=1$.

\paragraph{Importance sampling from $\pi^+$ to cover high-reward regimes.}
Revisiting Eq.~\eqref{eq:Dt}, we find one problem: sampling $\vx_t$ from $\pi_{i,t}^+$ is impractical since we cannot directly sample from the $\pi_{i}^+$ of Eq.~\eqref{eq:pi+def}. One seemingly plausible solution is to sample from $\pi^{\rm old}$ instead: since $\rho_{i,t}=\pi_{i,t}^+/\pi_t^{\rm old}$, we have $\E_{\pi_t^{\mathrm{old}}}\!\left[\rho_{i,t}(\vx_t)^\alpha\right] \;=\; \E_{\pi_{i,t}^+}\!\left[\rho_{i,t}(\vx_t)^{\alpha-1}\right]$. This means we draw $\widetilde{\vx}_0\sim\pi^{\rm old}(\cdot\mid\vc)$, forward-noise it to $\vx_t\sim\pi^{\rm old}_t$ by $\vx_t = (1-\sigma_t)\widetilde{\vx}_0 + \sigma_t\vepsilon$, and roll out multiple $\vx_0$ from $\vx_t$ under $\pi^{\mathrm{old}}$. Here, $\vx_0$ is used to estimate $\mu_{i,t}(\vx_t)$ and $\widetilde{\vx}_0$ is used to estimate $Z_i$. 
However, if the current policy ($\pi_{\rm old}$) often produces low-quality outputs, then $r(\vx_0,\vc)\approx r(\widetilde{\vx}_0,\vc)$, so $\rho_{i,t}(\vx_t)=\mu_{i,t}(\vx_t)/Z_i\approx1$ and $D_{i,t}\approx 0$: the high-reward regions are unestimated. The leftmost column of Fig.~\ref{fig:renyi-omega-combined} shows exactly this: under $\pi^{\rm old}$ the estimated $D_{i,t}$ curves for \textsc{ClipScore}, \textsc{HPSv2}, and \textsc{PickScore} are close to $0$ and significantly lower than their $\pi^+$ counterparts (App.~\ref{app:psampling-ablation}).

Therefore, we propose to use a single stronger external generator $\pi^+$ as a common proposal for all of the $\pi_{i}^+$: we draw $\widetilde{\vx}_0\sim\pi^+(\cdot\mid\vc)$, forward-noise it to $\vx_t\sim\pi_t^+$ by $\vx_t = (1-\sigma_t)\widetilde{\vx}_0 + \sigma_t\vepsilon$, and roll out $\vx_0$ from $\vx_t$ under $\pi^{\mathrm{old}}$. Since we assume $\pi^+$ generates higher-quality outputs than $\pi^{\rm old}$, $\rho_{i,t}(\vx_t)=\mu_{i,t}(\vx_t)/Z_i$ will decrease to 1 as $t$ increases. 
Because the ratio $\pi_{i,t}^+/\pi_t^+$ is unknown, this is not an exact importance-sampling estimator of $D_{i,t}$ but a surrogate log-moment. App.~\ref{app:positive-policy-error} analyzes the error the surrogate induces: when $\mathrm{KL}(\pi^+\|\pi_i^+)$ is small, the error in the discriminability gain is $O(\sqrt{\mathrm{KL}(\pi^+\|\pi_i^+)})$ and propagates to the Sinkhorn weights. Furthermore, replacing $\pi^+$ with $\pi^{\mathrm{old}}$ self-rollouts yields far noisier curves in App.~\ref{app:psampling-ablation}, while choosing a different $\pi^+$, either \textsc{GPT Image 1.5} or \textsc{Nano Banana Pro}, leaves the estimated curves essentially unchanged in App.~\ref{app:surrogate-sensitivity}. 

\paragraph{Denominator cancellation.}
Putting $\rho_{i,t}(\vx_t)=\mu_{i,t}(\vx_t)/Z_i$ into Eq.~\eqref{eq:Dt}, we have $D_{i,t} = \frac{1}{\alpha-1}\log\E_{\pi_{i,t}^+}[\mu_{i,t}(\vx_t)^{\alpha-1}] - \log Z_i$. Because only the \emph{gain} $\Delta D_{i,t} = D_{i,t-1} - D_{i,t}$ is needed, $\log Z_i$ cancels entirely:
\begin{equation}
\label{eq:gain-no-denom}
\boxed{\,\Delta D_{i,t} = \frac{1}{\alpha-1}\Big(\log\E_{\pi_{i, t-1}^+}\!\big[\mu_{i,t-1}(\vx_{t-1})^{\alpha-1}\big] - \log\E_{\pi_{i,t}^+}\!\big[\mu_{i,t}(\vx_t)^{\alpha-1}\big]\Big).\,}
\end{equation}
Replacing the unknown positive marginals by the common proposal $\pi_t^+$ leaves the quantity we actually compute,
\begin{equation}
\label{eq:surrogate-gain}
\widetilde{\Delta D}_{i,t} = \frac{1}{\alpha-1}\Big(\log\E_{\pi_{t-1}^{+}}\!\big[\mu_{i,t-1}(\vx_{t-1})^{\alpha-1}\big] - \log\E_{\pi_t^{+}}\!\big[\mu_{i,t}(\vx_t)^{\alpha-1}\big]\Big),
\end{equation}
which coincides with Eq.~\eqref{eq:gain-no-denom} only when $\pi_t^+ = \pi_{i,t}^+$. Only $\mu_{i,t}$ has to be estimated, by rollouts under $\pi^{\mathrm{old}}$, and $Z_i$ is never formed. Alg.~\ref{alg:estimation} is the resulting estimator, in compact form; App.~\ref{app:estimation-alg} writes out every index. For readability in the empirical sections we write the output $\widehat{\widetilde{\Delta D}_{i,t}}$ simply as $\Delta D_{i,t}$, and identities involving $D_\alpha$ refer to the population quantity unless stated otherwise.

\begin{algorithm}[t]
\caption{Per-step R\'enyi discriminability gain proxies $\widehat{\widetilde{\Delta D}_{i,t}}$ from $\pi^+$ samples}
\label{alg:estimation}
\small
\begin{algorithmic}[1]
\Require Prompt $\vc$, policy $\pi^{\mathrm{old}}$, stronger external generator $\pi^+(\cdot\mid\vc)$ used as a common proposal, rewards $r_{1:m}$, order $\alpha>0$ with $\alpha\neq1$, samples $N$, rollouts $K$, schedule $\{\sigma_t\}_{t=0}^T$
\Ensure Per-step gain proxies $\widehat{\widetilde{\Delta D}_{i,t}}$ for every reward $i$ and step $t$
\For{$n = 1,\dots,N$}
    \State Draw $\widetilde{\vx}_0^{(n)}\!\sim\pi^+(\cdot\mid\vc)$; set $\vx_t^{(n)}\!\gets(1-\sigma_t)\widetilde{\vx}_0^{(n)}\!+\sigma_t\vepsilon_t$, $\vepsilon_t\!\sim\!\mathcal N(0,I)$, $t=0,\dots,T$ \Comment{$\vx_t^{(n)}\!\sim\pi_t^+$}
    \State $\hat\mu_{i,t}^{(n)} \gets \frac{1}{K_t}\sum_{k\leq K_t} r_i(\vx_0^{(n,t,k)},\vc)$ from $K_t$ rollouts $\vx_t^{(n)}\!\to\!\vx_0^{(n,t,k)}$ under $\pi^{\mathrm{old}}$ ($K_0{=}1$, else $K$)
\EndFor
\State $\widehat{\widetilde D_{i,t}} \gets \frac{1}{\alpha-1}\log\!\big(\max\{10^{-10},\frac1N\sum_{n} [\hat\mu_{i,t}^{(n)}]^{\alpha-1}\}\big)$ for all $i,t$ \Comment{numerical floor}
\State \Return $\widehat{\widetilde{\Delta D}_{i,t}} \gets \big[\widehat{\widetilde D_{i,t-1}} - \widehat{\widetilde D_{i,t}}\big]_+$, $t = 1,\dots,T$ \Comment{Eq.~\eqref{eq:surrogate-gain}, clipped}
\end{algorithmic}
\end{algorithm}

\subsection{Sinkhorn projection}
\label{sec:sinkhorn}

The gains $\Delta D_{i,t}$ measure the utility of each reward at every step, but they do not yet constitute valid weights: they neither satisfy the budget $\bm{\lambda}$ nor maintain equal per-step total weights across $t$. We formulate the two requirements as marginal constraints and single out a unique matrix satisfying them by entropic projection. We ablate the choice of equal per-step total weights in Sec.~\ref{sec:sinkhorn-ablation}.

We first define the \emph{mean-normalized gain}
\begin{equation}
\label{eq:gain-bar}
\overline{\Delta D}_{i,t} \;:=\; \frac{\Delta D_{i,t}}{\frac{1}{T}\sum_{s=1}^{T}\Delta D_{i,s}},
\qquad\text{so that}\qquad \frac{1}{T}\sum_{t}\overline{\Delta D}_{i,t} \;=\; 1 \;\;\text{for every } i.
\end{equation}
Because the gains telescope (Eq.~\eqref{eq:telescope}), the denominator is $D_{i,0}/T$, so the division removes the reward's total discriminability and keeps only its shape in $t$ to better capture the intra-timestep relation. 

From the normalized gains we build the \emph{affinity kernel}
\begin{equation}
\label{eq:kernel}
{\;K_{i,t} \;:=\; \exp\!\big(\overline{\Delta D}_{i,t}\big) \;}
\end{equation}
which is largest at the steps where reward $i$ contributes most of its discriminability and carries no scale parameter of its own, since $\overline{\Delta D}_{i,t}$ is already mean-normalized. Theorem~\ref{thm:debruijn} gives $\Delta D_{i,t}\ge0$ in population; in finite samples we clip, using $[\Delta D_{i,t}]_+$, so a step estimated as negative enters at $K_{i,t}=1$, the smallest value the kernel takes.

Recall the two marginal constraints from Eq.~\eqref{eq:two-level-master}. Writing $\vone_m,\vone_T$ for all-ones vectors, the feasible set is the \emph{transportation polytope}
\begin{equation}
\label{eq:polytope}
\mathcal U(\bm{\lambda}) \;:=\; \Big\{\,\bm{W}\in\mathbb R_{\geq0}^{m\times T} \;:\; \bm{W}\vone_T = \bm{\lambda},\;\; \bm{W}^\top\vone_m = \tfrac{1}{T}\vone_T \,\Big\}.
\end{equation}
The set is never empty, since the rank-one matrix $\lambda_i/T$ always belongs to it. Among its elements we want the $\mathrm{KL}$ projection of $\bm{K}$ onto the polytope,
\begin{equation}
\label{eq:sinkhorn-problem}
{\;\bm{W}^\star \;=\; \argmin_{\bm{W}\in\mathcal U(\bm{\lambda})}\; \mathrm{KL}\big(\bm{W} \,\|\, \bm{K}\big) \;=\; \argmin_{\bm{W}\in\mathcal U(\bm{\lambda})}\;\sum_{i,t} W_{i,t}\log\frac{W_{i,t}}{K_{i,t}} - W_{i,t} + K_{i,t},\;}
\end{equation}
which, substituting Eq.~\eqref{eq:kernel}, is an entropy-regularized transport problem with cost $C_{i,t} := -\overline{\Delta D}_{i,t}$ and unit regularization strength:
\begin{equation}
\label{eq:ot-form}
\bm{W}^\star \;=\; \argmin_{\bm{W}\in\mathcal U(\bm{\lambda})}\;\sum_{i,t}\Big[\,C_{i,t}\,W_{i,t} \;+\; W_{i,t}\log W_{i,t}\,\Big].
\end{equation}
Since every $K_{i,t}>0$, the objective is strictly convex and the minimizer is unique. Setting the Lagrangian gradient to zero gives the kernel rescaled by one factor per row and one per column,
\begin{equation}
\label{eq:sinkhorn-form}
{\;W_{i,t}^\star \;=\; a_i\;K_{i,t}\;b_t \;=\; a_i\cdot\exp\!\big(\overline{\Delta D}_{i,t}\big)} \cdot b_t,
\end{equation}
where $a_i$ carries inter-reward scaling, and $b_t$ the per-step normalization.
Sinkhorn's theorem~\citep{sinkhorn1967concerning,cuturi2013sinkhorn} guarantees that $\va\in\mathbb R^m_{>0}$ and $\vb\in\mathbb R^T_{>0}$ exist and are unique up to the trivial rescaling $(\va,\vb)\mapsto(c\,\va,\vb/c)$, and it finds them iteratively:
\begin{equation}
\label{eq:sinkhorn-iter}
a_i \;\leftarrow\; \frac{\lambda_i}{\sum_t K_{i,t}\,b_t},
\qquad\qquad
b_t \;\leftarrow\; \frac{1/T}{\sum_i K_{i,t}\,a_i},
\end{equation}
which converges linearly; in the log domain it needs $<100$ steps at $m\leq 4$, $T\leq 25$ (Alg.~\ref{alg:sinkhorn}, App.~\ref{app:sinkhorn-solver}).

A reward whose gain is flat in $t$ has $\overline{\Delta D}_{i,t}\equiv1$, since Eq.~\eqref{eq:gain-bar} fixes the mean at $1$; if this holds for every reward, the kernel is rank one and the entropic optimum is $W_{i,t}^\star = \lambda_i/T$, so the static baseline is exactly the case of flat estimated curves (App.~\ref{app:sinkhorn-cases}).

\subsection{Empirical gain curves on SD3.5-Medium}
\label{sec:empirical-omega}

\paragraph{Setup.}
We instantiate Alg.~\ref{alg:estimation} on SD3.5-Medium with $\alpha{=}2$, $80$ prompts from the training datasets, $N{=}8$ samples per prompt, and $K{=}16$ rollouts per $\vx_t$. We approximate $\pi^+$ by \textsc{GPT Image 1.5}, VAE-encode its images into latent space, forward-noise to each step, and run batched rollouts under $\pi^{\mathrm{old}}$. We plot seven rewards: \textsc{ClipScore}, \textsc{HPSv2}, \textsc{PickScore}, \textsc{Aesthetic}, and \textsc{ImageReward} jointly on this shared $T{=}10$ grid, plus \textsc{OCR} and \textsc{GenEval} on their own prompt sets at the native $T{=}25$ schedule they are later trained under. Each reward's architecture, checkpoint, and range are in Sec.~\ref{sec:exp-setup} and App.~\ref{app:rewards}; App.~\ref{app:alpha-sweep} conducts an ablation study on $\alpha$.

\paragraph{Results.}
In Fig.~\ref{fig:renyi-omega}, the estimated curves are strongly reward-specific; the horizontal axis is the diffusion timestep, from the cleanest step $t{=}1$ (left) to the noisiest $t{=}T$ (right). \textsc{GenEval}, \textsc{OCR}, and \textsc{ImageReward} peak the most sharply, concentrating their gain toward the noisy end, while the per-step gain of the others stays close to uniform. The first two are rule-based and, like \textsc{ImageReward}, can score a sample once its global structure has emerged. The differences are large enough to matter: under uniform $\bm{\lambda}$, the aggregate demand $\sum_i\lambda_i\overline{\Delta D}_{i,t}$ varies by $5.4\times$ across the steps $2\le t\le T$, and still by $4.5\times$ once \textsc{GenEval} is excluded. The right panel shows the Sinkhorn projection (Eq.~\eqref{eq:sinkhorn-iter}) over the four rewards in the \textsc{OCR} setting, correcting this imbalance by shifting each reward's weight toward its own high-gain steps while both marginals stay exact.

\begin{figure}[t]
\centering
\includegraphics[width=\linewidth]{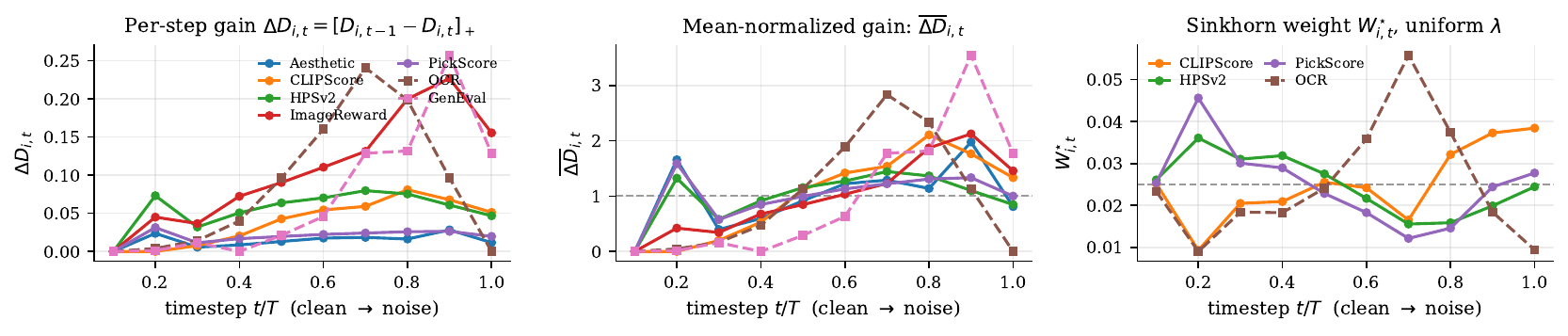}
\caption{Gain curves and the weight matrix they project to. \textbf{Left:} per-step gain $\Delta D_{i,t}$. \textbf{Center:} mean-normalized gain $\overline{\Delta D}_{i,t}$ (Eq.~\eqref{eq:gain-bar}), the exponent of the affinity kernel (Eq.~\eqref{eq:kernel}). \textbf{Right:} the Sinkhorn matrix $W_{i,t}^\star$ (Eq.~\eqref{eq:sinkhorn-form}) at uniform $\lambda_i{=}1/m$, solved over the four rewards that are trained together in the \textsc{OCR} setting; both marginals hold exactly, every column summing to $1/T$ and every row to $\lambda_i$. In all panels the horizontal axis is $t/T$, running from the cleanest step $t{=}1$ (left) to the noisiest $t{=}T$ (right). \textsc{OCR} and \textsc{GenEval} (dashed) are resampled onto the shared $T{=}10$ grid from their native $T{=}25$ curves, so the right panel is illustrative: the training-stage kernel is solved at $T{=}25$ on \textsc{OCR}'s own grid (Sec.~\ref{sec:exp-setup}). }
\label{fig:renyi-omega}
\end{figure}

\section{Experiments}
\label{sec:pareto-plan}
Sec.~\ref{sec:exp-setup} introduces the setup. Sec.~\ref{sec:exp-train} evaluates the training rewards on a held-out split, and Sec.~\ref{sec:exp-heldout} the same checkpoints under held-out judges, to show generalizable improvement. Sec.~\ref{sec:MMRBv2} rescores those images with an independent LLM-as-a-Judge, and Sec.~\ref{sec:sinkhorn-ablation} ablates the Sinkhorn projection.

\subsection{Training setup}
\label{sec:exp-setup}
\paragraph{Reward models.} We evaluate \ourmethod{} with nine reward models in two distinct roles. Five are \emph{training rewards}: \textsc{ClipScore}~\citep{hessel2021clipscore}, \textsc{HPSv2}~\citep{wu2023hpsv2}, \textsc{PickScore}~\citep{kirstain2023pick}, \textsc{OCR}, and \textsc{GenEval}~\citep{ghosh2023geneval}. To evaluate generalization, four \emph{held-out judges} score the checkpoints: \textsc{Aesthetic}~\citep{schuhmann2022aesthetics}, \textsc{ImageReward}~\citep{xu2023imagereward}, \textsc{HPSv3}~\citep{ma2025hpsv3}, and \textsc{UnifiedReward-2}~\citep{wang2025unified}. App.~\ref{app:rewards} details architectures, checkpoints, and output ranges for all rewards.

\paragraph{Training data and reward chain.} Following DiffusionNFT~\citep{diffusionnft}, training runs in two stages. The \textbf{warmup stage} jointly optimizes \textsc{ClipScore}, \textsc{HPSv2}, and \textsc{PickScore} for $120$ steps from the base model on Pick-a-Pic prompts ($25{,}432$ train/$2{,}048$ test); every later run resumes from its checkpoint. The \textbf{training stage-OCR} then adds an \textsc{OCR} reward on text-rendering prompts ($19{,}652$ train/$1{,}017$ test) for $60$ further steps ($m{=}4$). The \textbf{training stage-GenEval} branches from the same warmup parent on compositional \textsc{GenEval} prompts ($50{,}000$ train/$2{,}211$ test), swapping \textsc{OCR} for a rule-based \textsc{GenEval} reward on an identical schedule.

\paragraph{Training configuration.} Across all stages, we fine-tune LoRA adapters of Stable Diffusion 3.5-Medium. We use the DiffusionNFT objective with $\beta{=}0.1$, rolling out $T{=}25$ sampling steps per iteration at an effective batch size of $1152$. We run the warmup stage at a balanced $\bm{\lambda}\!=\!(1,1,1)$ on three rewards: \textsc{ClipScore}, \textsc{HPSv2}, and \textsc{PickScore}. Then, we sweep each training-stage setting's row marginals over $\bm{\lambda} \in \{(1,1,1,1), (1,1,1,2), (1,1,2,1), (1,2,1,1), (2,1,1,1)\}$, the uniform budget together with each coordinate doubled in turn, where $\bm{\lambda}$ is ordered (\textsc{ClipScore}, \textsc{HPSv2}, \textsc{PickScore}, \textsc{OCR}/\textsc{GenEval}). Every $\bm{\lambda}$ is written unnormalized as an integer ratio between rewards, and is divided by its own sum before use to satisfy $\sum_i\lambda_i=1$. Each budget gives one matched pair of \ourmethod{} against a static baseline, and both settings are repeated end-to-end under $3$ independent seeds, giving $15$ matched (budget, seed) pairs per setting. For our weights configuration, we use Alg.~\ref{alg:estimation} with $\alpha{=}2$ for the kernel and Alg.~\ref{alg:sinkhorn} for the Sinkhorn projection. The training objective is detailed in App.~\ref{sec:background} and the remaining hyperparameters in App.~\ref{app:hparams}.

\paragraph{Evaluation protocol.} We evaluate at three levels: the \textbf{training rewards} on a held-out dataset (Sec.~\ref{sec:exp-train}); \textbf{held-out judge models} (Sec.~\ref{sec:exp-heldout}); and \textbf{LLM-as-a-Judge} (Sec.~\ref{sec:MMRBv2}). All three score each run's final checkpoint at matched seeds, $512^2$ resolution, CFG $4.5$, and $40$ steps; we report means over each training-stage pair-set and win rates over its matched pairs.

\subsection{Training rewards}

\label{sec:exp-train}
\begin{table}[!htbp]
\centering
\small
\setlength{\tabcolsep}{1.5pt}
\caption{Training-reward scores at the final checkpoint, reported as base~/~static~/~\oursshort{}, where \emph{base} is the untuned SD3.5-Medium, \emph{static} is static reweighting by $\bm{\lambda}$, and Aggregate $r$ is the sum of the individual rewards weighted by the integer $\bm{\lambda}$. Both settings were trained end-to-end (warmup stage + training-stage sweep, both methods) under $3$ seeds; entries are the mean of the three seed-level means, standard deviation as a subscript, and \emph{win} pools all $15$ matched (budget, seed) pairs. Bold marks the better of static and \oursshort{}, and both when the two agree at the reported precision. Per-$\bm{\lambda}$ results are in App.~\ref{app:preagg-geneval}.}
\label{tab:trainreward}
\begin{tabular}{@{}lcccc@{}}
\toprule
& \multicolumn{2}{c}{\textbf{Training stage-OCR} ($5$ budgets $\times$ $3$ seeds)} & \multicolumn{2}{c}{\textbf{Training stage-GenEval} ($5$ budgets $\times$ $3$ seeds)} \\
\cmidrule(lr){2-3}\cmidrule(lr){4-5}
Training reward & base~/~static~/~\oursshort{} & win & base~/~static~/~\oursshort{} & win \\
\midrule
Aggregate $r$       & 1.642~/~2.867$_{\pm0.026}$~/~\textbf{2.920}$_{\pm0.028}$ & 11/15 & 1.784~/~\textbf{2.881}$_{\pm0.021}$~/~2.878$_{\pm0.015}$ & 7/15 \\
\textsc{ClipScore}  & 0.271~/~0.313$_{\pm0.002}$~/~\textbf{0.320}$_{\pm0.002}$ & 13/15 & 0.240~/~\textbf{0.299}$_{\pm0.001}$~/~\textbf{0.299}$_{\pm0.001}$ & 6/15 \\
\textsc{HPSv2}      & 0.195~/~0.288$_{\pm0.008}$~/~\textbf{0.303}$_{\pm0.004}$ & 12/15 & 0.213~/~\textbf{0.313}$_{\pm0.005}$~/~0.312$_{\pm0.002}$ & 5/15 \\
\textsc{PickScore}  & 0.773~/~0.888$_{\pm0.002}$~/~\textbf{0.899}$_{\pm0.003}$ & 15/15 & 0.790~/~\textbf{0.911}$_{\pm0.004}$~/~0.908$_{\pm0.003}$ & 6/15 \\
Target reward       & 0.129~/~0.893$_{\pm0.014}$~/~\textbf{0.900}$_{\pm0.032}$ & 9/15 & 0.244~/~0.873$_{\pm0.011}$~/~\textbf{0.874}$_{\pm0.007}$ & 9/15 \\
\bottomrule
\end{tabular}
\end{table}

Tab.~\ref{tab:trainreward} reports each training reward on a held-out dataset, averaged over the five budgets and three seed replicates, with \emph{win} counting the matched pairs in which \ourmethod{} beats static. \ourmethod{} raises the aggregate score from $2.867$ to $2.920$ on average ($+0.053_{\pm0.049}$ across seeds) and wins $11$ of the $15$ matched (budget, seed) pairs. The gain is not obtained by sacrificing one objective for another: all four component rewards improve on average, with \textsc{PickScore} winning all $15$ pairs and \textsc{ClipScore} and \textsc{HPSv2} winning $13$ and $12$. The target \textsc{OCR} reward carries the largest seed variance ($+0.007_{\pm0.046}$, $9/15$). Adaptive timing therefore uses a fixed multi-reward budget more efficiently, rather than merely changing the trade-off encoded by $\bm{\lambda}$. On the GenEval setting the two methods instead tie: every training reward shifts by less than one seed standard deviation (aggregate $-0.003_{\pm0.031}$, target \textsc{GenEval} $+0.001_{\pm0.017}$) and win rates are near chance ($5$ to $9$ of $15$). That setting therefore acts as a matched-training-reward control, where Sec.~\ref{sec:exp-heldout} still finds \ourmethod{} preferred by every held-out judge.

\subsection{Held-out judges}
\label{sec:exp-heldout}

\begin{table}[!htbp]
\centering
\small
\setlength{\tabcolsep}{2.5pt}
\caption{Held-out judge scores at the final checkpoint, reported as base~/~static~/~\oursshort{}. Both settings cover $5$ budgets $\times$ $3$ independent training seeds, and entries, subscripts, \emph{win}, and the bolding rule are as in Tab.~\ref{tab:trainreward}; the \textsc{HPSv3} row is at two decimals because its scale is an order of magnitude larger. Per-$\bm{\lambda}$ results are in App.~\ref{app:preagg-geneval}.}
\label{tab:panels}
\begin{tabular}{@{}lcccc@{}}
\toprule
& \multicolumn{2}{c}{\textbf{Training stage-OCR} ($15$ pairs)} & \multicolumn{2}{c}{\textbf{Training stage-GenEval} ($15$ pairs)} \\
\cmidrule(lr){2-3}\cmidrule(lr){4-5}
Held-out model & base~/~static~/~\oursshort{} & win & base~/~static~/~\oursshort{} & win \\
\midrule
\textsc{ImageReward}     & 1.140~/~1.422$_{\pm0.037}$~/~\textbf{1.482}$_{\pm0.030}$ & 13/15 & 1.024~/~1.198$_{\pm0.165}$~/~\textbf{1.289}$_{\pm0.122}$ & 9/15 \\
\textsc{Aesthetic}       & 5.388~/~5.486$_{\pm0.063}$~/~\textbf{5.578}$_{\pm0.009}$ & 12/15 & 5.949~/~5.864$_{\pm0.120}$~/~\textbf{5.968}$_{\pm0.102}$ & 10/15 \\
\textsc{HPSv3}           & 12.69~/~12.74$_{\pm0.18}$~/~\textbf{13.14}$_{\pm0.49}$ & 11/15 & 9.19~/~8.02$_{\pm0.95}$~/~\textbf{8.58}$_{\pm0.56}$ & 10/15 \\
\textsc{UnifiedReward-2} & 3.015~/~3.154$_{\pm0.025}$~/~\textbf{3.176}$_{\pm0.040}$ & 11/15 & 2.682~/~2.671$_{\pm0.085}$~/~\textbf{2.705}$_{\pm0.046}$ & 9/15 \\
\bottomrule
\end{tabular}
\end{table}

As shown in Tab.~\ref{tab:panels}, on the OCR setting all four held-out judges score \ourmethod{} higher on average, winning $13$, $12$, $11$, and $11$ of the $15$ matched (budget, seed) pairs. For the GenEval setting, although the two methods are tied there on every reward being optimized (Sec.~\ref{sec:exp-train}), all four held-out judges still favor \ourmethod{} (largest, \textsc{HPSv3} $+0.56_{\pm0.41}$), winning $9$ or $10$ of the $15$ pairs. A held-out gain under an equal training reward is consistent with generalization rather than overfitting, and on the OCR setting the improved text rendering enhances rather than degrades broader image quality. The two settings differ against the untuned base model, however: the OCR setting clears it on all four judges under both methods, whereas in the GenEval setting static falls below it on \textsc{Aesthetic}, \textsc{HPSv3}, and \textsc{UnifiedReward-2}. Optimizing compositional correctness therefore costs generic visual quality, and \ourmethod{} recovers that cost on every judge, fully on \textsc{Aesthetic} and \textsc{UnifiedReward-2} and partly on \textsc{HPSv3}.

\subsection{LLM-as-a-Judge}
\label{sec:MMRBv2}

To assess general preference closer to real use, we use the MMRBv2 eval script~\citep{hu2026multimodal} to compare the \ourmethod{} checkpoints against static on a fixed $1{,}000$-prompt subset of the held-out OCR set. A multimodal LLM judge (\texttt{gemini-3.5-flash}, temperature $0$) picks the better of the two images per prompt according to a fixed rubric, and every pair is judged twice with positions swapped to cancel position bias; we also report the rubric's faithfulness and aesthetics criteria. App.~\ref{app:MMRBv2-details} gives the full prompt and the debiasing procedure.

\begin{table}[!htbp]
\centering
\small
\setlength{\tabcolsep}{3.5pt}
\caption{MMRBv2 pairwise evaluation of \ourmethod{} against static on the same subset as Tab.~\ref{tab:panels}, position-debiased over both presentation orders. Every entry is \ourmethod{}'s win rate, overall and under the two rubric criteria, with $0.5$ indicating parity; $\bm{\lambda}$ is ordered (\textsc{ClipScore}, \textsc{HPSv2}, \textsc{PickScore}, \textsc{OCR}).}
\label{tab:MMRBv2-ocr}
\begin{tabular}{@{}lcccccc@{}}
\toprule
$\bm{\lambda}$ & $(1,1,1,1)$ & $(1,1,1,2)$ & $(1,1,2,1)$ & $(1,2,1,1)$ & $(2,1,1,1)$ & \textbf{mean} \\
\midrule
Overall      & 0.6538 & 0.5580 & 0.4780 & 0.5365 & 0.6325 & \textbf{0.5718} \\
Faithfulness & 0.6550 & 0.5570 & 0.5130 & 0.5655 & 0.6094 & \textbf{0.5800} \\
Aesthetics   & 0.6438 & 0.5397 & 0.4736 & 0.4674 & 0.6242 & \textbf{0.5497} \\
\bottomrule
\end{tabular}
\end{table}

As shown in Tab.~\ref{tab:MMRBv2-ocr}, \ourmethod{} obtains a mean overall win rate of $0.572$ and is preferred at four of the five budgets, most strongly at $(1,1,1,1)$ and $(2,1,1,1)$ ($0.654$ and $0.633$), where both faithfulness and aesthetics improve. App.~\ref{app:qualitative-gallery} shows paired generations behind these judgments and their selection protocol.
\subsection{Ablation study: the Sinkhorn projection}
\label{sec:sinkhorn-ablation}
We ablate the Sinkhorn projection of Sec.~\ref{sec:sinkhorn}, which enforces a uniform total budget at each step, by training \emph{no-Sink R\'enyi} variants that keep the adaptive kernel but omit the column constraint: $W^\star$ becomes the row-normalized kernel $W_{i,t} = \lambda_i K_{i,t}/\sum_{s} K_{i,s}$. Every reward still spends exactly its budget $\lambda_i$ and the loss still consumes $T\!\cdot\!W_{i,t}$, but $\sum_i W_{i,t}$ is free, so each reward rides its own gain shape. All other training choices match the \textsc{OCR} experiment.

\begin{table}[!htbp]
\centering
\small
\setlength{\tabcolsep}{3.5pt}
\caption{Aggregate training-reward score for each $\bm{\lambda}=(\text{clip},\text{hps},\text{pick},\text{ocr})$ budget in the training stage-OCR experiment, comparing static weighting, the adaptive kernel alone (no-Sink R\'enyi), and the full method. Scores are on the held-out OCR dataset at the last checkpoint (step $180$), single matched training seed.}
\label{tab:sinkhorn-ablation-perlam}
\begin{tabular}{@{}lccccc@{}}
\toprule
$\bm{\lambda}$ & $(1,1,1,1)$ & $(1,1,1,2)$ & $(1,1,2,1)$ & $(1,2,1,1)$ & $(2,1,1,1)$ \\
\midrule
static                  & 2.3848 & 3.2476 & 3.2948 & 2.5597 & 2.6980 \\
no-Sink R\'enyi         & 2.4242 & 3.3514 & \textbf{3.3781} & 2.6721 & 2.6969 \\
\ourmethod{} (Sinkhorn) & \textbf{2.4650} & \textbf{3.3564} & 3.3564 & \textbf{2.7960} & \textbf{2.7038} \\
\bottomrule
\end{tabular}
\end{table}

As shown in Tab.~\ref{tab:sinkhorn-ablation-perlam}, no-Sink R\'enyi beats static on four of five budgets, and the full method beats no-Sink R\'enyi on four of five. Averaged across budgets, the adaptive kernel accounts for $+0.068$ of \ourmethod{}'s total $+0.099$ gain over static, and the Sinkhorn projection for the remaining $+0.031$ (this seed's gap, not the three-seed mean of Tab.~\ref{tab:trainreward}; both parts of the decomposition share the seed). Reward-specific timing therefore identifies useful optimization steps, while the shared column budget keeps the resulting weights from concentrating on a few of them.

\section{Conclusion}
\label{sec:conclusion}
Multi-reward diffusion fine-tuning must decide both \emph{how much} each reward matters and \emph{when} it should act. \ourmethod{} separates these decisions: the budget $\bm{\lambda}$ fixes each reward's total contribution, and its R\'enyi discriminability curve distributes it across timesteps.

On SD3.5-Medium, this temporal reallocation improves the OCR setting's aggregate training objective across five budgets and three seeds, with all four training rewards increasing on average, and it transfers: every held-out judge improves over static, including on the GenEval setting where the two methods tie on the training rewards, and the pairwise LLM judge prefers \ourmethod{} at four of five budgets. \emph{When} a reward is applied can therefore matter as much as \emph{how much} weight it receives. We discuss the limitations of our work in App.~\ref{sec:limitations}.

\bibliographystyle{iclr2027_conference}
\bibliography{iclr2027_conference}

\newpage
\appendix
\section{Supplementary theory}
In App.~\ref{sec:background}, we provide notations and basic properties on rectified flow matching and DiffusionNFT. In App.~\ref{app:proof-debruijn}, we prove Theorem~\ref{thm:debruijn}. In App.~\ref{app:positive-policy-error}, we bound the error induced by replacing each reward-specific positive policy $\pi_i^+$ with the common surrogate $\pi^+$, and show that it is $O(\sqrt{\mathrm{KL}(\pi^+\|\pi_i^+)})$ on the discriminability gain and propagates to the Sinkhorn weights.
\subsection{Background: rectified flow matching and DiffusionNFT}
\label{sec:background}

\paragraph{Rectified flow matching.}
All models in this paper are rectified-flow generators~\citep{liu2022flow,lipman2022flow,esser2024scaling}. A clean sample $\vx_0\sim\pi(\cdot\mid\vc)$ and noise $\vepsilon\sim\mathcal N(0,I)$ are connected by the straight interpolation
\begin{equation}
\label{eq:rf-path}
\vx_t \;=\; \alpha_t\,\vx_0 + \sigma_t\,\vepsilon, \qquad \alpha_t = 1-\sigma_t, \qquad \sigma_0 = 0,\quad \sigma_T = 1,
\end{equation}
so $t{=}0$ is clean data and $t{=}T$ is pure noise. Here $\alpha_t$ is the interpolation coefficient, distinct from the R\'enyi order $\alpha$ of Sec.~\ref{sec:Dt} and from the mixture coefficient $\alpha(\vx_t)$ defined below. Differentiating Eq.~\eqref{eq:rf-path} gives the target velocity $\vv = \dot\alpha_t\,\vx_0 + \dot\sigma_t\,\vepsilon$, which for the linear path is simply $\vv = \vepsilon - \vx_0$. A velocity model $\vv_\theta$ is fit by regression on this target,
\begin{equation}
\label{eq:fm-loss}
\mathcal L_{\mathrm{FM}}(\theta) \;=\; \E_{\vc,\;\vx_0,\;\vepsilon,\;t}\Big[\,\big\|\vv_\theta(\vx_t,\vc,t) - \vv\big\|_2^2\,\Big],
\end{equation}
and sampling integrates ${\rm d}\vx = \vv_\theta(\vx_t,\vc,t)\,{\rm d}\sigma_t$ backward from $t{=}T$ to $t{=}0$ over a discrete schedule $\{\sigma_t\}_{t=0}^T$ of $T$ steps. Two features of Eq.~\eqref{eq:fm-loss} matter for what follows: the expectation over $t$ makes every timestep a separate regression problem, so a per-timestep coefficient can be attached to each without changing the estimator; and $\vv$ depends on the data only through $\vx_0$, so tilting the data distribution by a reward moves the target while leaving the forward path Eq.~\eqref{eq:rf-path} untouched. The latter is what lets Theorem~\ref{thm:debruijn} apply to the reward-tilted marginals in Sec.~\ref{sec:Dt}.

\paragraph{DiffusionNFT.}
We briefly review the DiffusionNFT framework~\citep{diffusionnft}, which forms the basis of our multi-reward training objective. Theorems~\ref{thm:nft-direction} and~\ref{thm:nft-optimization} below are from DiffusionNFT~\citep{diffusionnft}; we restate them here, without proof, for the notation our objective builds on.

\paragraph{Positive Policy.}
Let $\pi^{\mathrm{old}}(\vx_0\mid\vc)$ denote the current, or old, policy and $r(\vx_0,\vc) := p(\vo=1\mid\vx_0,\vc)$ the optimality probability. The \emph{positive policy} is defined as the old policy conditioned on optimality:
\begin{equation}
\label{eq:pi-plus}
\pi^+(\vx_0\mid\vc) \;:=\; \pi^{\mathrm{old}}(\vx_0\mid\vo=1,\vc) \;=\; \frac{r(\vx_0,\vc)}{p_{\pi^{\mathrm{old}}}(\vo=1\mid\vc)}\,\pi^{\mathrm{old}}(\vx_0\mid\vc).
\end{equation}
With infinitely many samples from $\pi^{\mathrm{old}}$, $\pi^+$ is simply the distribution of the positive subset: the old policy restricted to its ``optimal'' outputs.

\paragraph{The Coefficient $\alpha(\vx_t)$.}

A key quantity linking the diffused positive marginal to the diffused old marginal is the scalar coefficient $\alpha(\vx_t)\in[0,1]$:
\begin{equation}
\label{eq:alpha}
\alpha(\vx_t) \;:=\; \frac{\pi_t^+(\vx_t\mid\vc)}{\pi_t^{\mathrm{old}}(\vx_t\mid\vc)}\cdot\E_{\pi^{\mathrm{old}}(\vx_0\mid\vc)}[r(\vx_0,\vc)].
\end{equation}
It comes out of a posterior decomposition: the old posterior splits into a mixture of the positive and negative posteriors, weighted by $\alpha(\vx_t)$ and $1-\alpha(\vx_t)$.

\begin{theorem}[Improvement direction,~\citep{diffusionnft}]
\label{thm:nft-direction}
Let $\vv^+$, $\vv^-$, and $\vv^{\mathrm{old}}$ be the velocity models for the policies $\pi^+$, $\pi^-$, and $\pi^{\mathrm{old}}$, respectively. The directional differences between these models are proportional:
\begin{equation}
\label{eq:Delta}
\Delta \;:=\; \big[1-\alpha(\vx_t)\big]\,\big[\vv^{\mathrm{old}}(\vx_t,\vc,t) - \vv^-(\vx_t,\vc,t)\big] \;=\; \alpha(\vx_t)\,\big[\vv^+(\vx_t,\vc,t) - \vv^{\mathrm{old}}(\vx_t,\vc,t)\big],
\end{equation}
where $0\leq\alpha(\vx_t)\leq 1$ is defined in Eq.~\eqref{eq:alpha}.
\end{theorem}
So $\Delta$ is an improvement direction in velocity space: moving from $\vv^{\mathrm{old}}$ toward $\vv^+$ is the same direction as moving away from $\vv^-$.

\begin{theorem}[Policy optimization,~\citep{diffusionnft}]
\label{thm:nft-optimization}
Consider the training objective
\begin{equation}
\label{eq:nft-loss}
\mathcal L(\theta) \;=\; \E_{\vc,\,\pi^{\mathrm{old}}(\vx_0\mid\vc),\,\vepsilon,\,t}\!\Big[r\,\|\vv_\theta^+(\vx_t,\vc,t) - \vv\|_2^2 \;+\; (1-r)\,\|\vv_\theta^-(\vx_t,\vc,t) - \vv\|_2^2\Big],
\end{equation}
where $\vv = \dot\alpha_t\,\vx_0 + \dot\sigma_t\,\vepsilon$ is the target velocity from the forward process, and the implicit positive and negative policies are parameterized as
\begin{align}
\label{eq:implicit-pos}
\vv_\theta^+(\vx_t,\vc,t) &\;:=\; (1-\beta)\,\vv^{\mathrm{old}}(\vx_t,\vc,t) + \beta\,\vv_\theta(\vx_t,\vc,t), \\
\label{eq:implicit-neg}
\vv_\theta^-(\vx_t,\vc,t) &\;:=\; (1+\beta)\,\vv^{\mathrm{old}}(\vx_t,\vc,t) - \beta\,\vv_\theta(\vx_t,\vc,t).
\end{align}
Given unlimited data and model capacity, the optimal solution satisfies
\begin{equation}
\label{eq:nft-optimal}
\vv_{\theta^*}(\vx_t,\vc,t) \;=\; \vv^{\mathrm{old}}(\vx_t,\vc,t) \;+\; \frac{2}{\beta}\,\Delta(\vx_t,\vc,t),
\end{equation}
where $\Delta$ is defined in Eq.~\eqref{eq:Delta}.
\end{theorem}
The optimum therefore sits along the reinforcement-guidance direction $\Delta$, at strength $2/\beta$ relative to $\vv^{\mathrm{old}}$, so the policy improves without any likelihood being estimated.

\subsection{Proof of Theorem~\ref{thm:debruijn}}
\label{app:proof-debruijn}

\begin{proof}
Let $\rho := \pi^a/\pi^b$ and $M_\alpha := \E_{\pi^b}[\rho^\alpha] = \int \rho^\alpha\,\pi^b\,{\rm d}\vx = \int (\pi^a)^\alpha\,(\pi^b)^{1-\alpha}\,{\rm d}\vx$. By definition, $D_\alpha = \frac{1}{\alpha-1}\log M_\alpha$, so
\begin{equation*}
\frac{{\rm d}}{{\rm d}t}D_\alpha \;=\; \frac{1}{\alpha-1}\cdot\frac{\dot M_\alpha}{M_\alpha}.
\end{equation*}
It suffices to show $\dot M_\alpha = -\tfrac{g^2}{2}\,\alpha(\alpha-1)\int \rho^\alpha\|\nabla\log\rho\|^2\,\pi^b\,{\rm d}\vx$.

\medskip\noindent\textbf{Step 1: Differentiating $M_\alpha$.} \;
Since $\rho = \pi^a/\pi^b$, we have $\partial_t\rho = (\partial_t\pi^a)/\pi^b - \rho\,(\partial_t\pi^b)/\pi^b$, so
\begin{equation*}
\dot M_\alpha \;=\; \int \big[\alpha\,\rho^{\alpha-1}\,(\partial_t\rho)\,\pi^b + \rho^\alpha\,\partial_t\pi^b\big]\,{\rm d}\vx \;=\; \alpha\!\int \rho^{\alpha-1}\,\partial_t\pi^a\,{\rm d}\vx \;-\; (\alpha{-}1)\!\int \rho^\alpha\,\partial_t\pi^b\,{\rm d}\vx.
\end{equation*}

\medskip\noindent\textbf{Step 2: Substituting the Fokker--Planck equation.} \;
Both $\pi^a$ and $\pi^b$ satisfy $\partial_t\pi = -\nabla\!\cdot(\vu\pi) + \tfrac{g^2}{2}\Delta\pi$, so
\begin{equation*}
\dot M_\alpha \;=\; \underbrace{\alpha\!\int \rho^{\alpha-1}\big[-\nabla\!\cdot(\vu\pi^a) + \tfrac{g^2}{2}\Delta\pi^a\big]\,{\rm d}\vx}_{=:\,I_a} \;-\; \underbrace{(\alpha{-}1)\!\int \rho^\alpha\big[-\nabla\!\cdot(\vu\pi^b) + \tfrac{g^2}{2}\Delta\pi^b\big]\,{\rm d}\vx}_{=:\,I_b}.
\end{equation*}

\medskip\noindent\textbf{Step 3: Drift terms cancel.} \;
For the drift part of $I_a$, integrate by parts (boundary terms vanish by decay at infinity):
\begin{equation*}
-\alpha\!\int \rho^{\alpha-1}\,\nabla\!\cdot(\vu\pi^a)\,{\rm d}\vx \;=\; \alpha\!\int \vu\pi^a\!\cdot\!\nabla(\rho^{\alpha-1})\,{\rm d}\vx \;=\; \alpha(\alpha{-}1)\!\int \vu\!\cdot\!\nabla\rho\;\rho^{\alpha-2}\pi^a\,{\rm d}\vx.
\end{equation*}
Using $\pi^a = \rho\pi^b$, this equals $\alpha(\alpha{-}1)\int \vu\!\cdot\!\nabla\rho\;\rho^{\alpha-1}\pi^b\,{\rm d}\vx$. For the drift part of $I_b$:
\begin{equation*}
(\alpha{-}1)\!\int \rho^\alpha\,\nabla\!\cdot(\vu\pi^b)\,{\rm d}\vx \;=\; -(\alpha{-}1)\!\int \vu\pi^b\!\cdot\!\nabla(\rho^\alpha)\,{\rm d}\vx \;=\; -\alpha(\alpha{-}1)\!\int \vu\!\cdot\!\nabla\rho\;\rho^{\alpha-1}\pi^b\,{\rm d}\vx.
\end{equation*}
Hence the drift contributions from $I_a$ and $I_b$ sum to zero.

\medskip\noindent\textbf{Step 4: Diffusion terms via integration by parts.} \;
It remains to evaluate the diffusive parts. For the $I_a$ diffusion term, apply Green's first identity (integration by parts twice):
\begin{align*}
\tfrac{g^2}{2}\,\alpha\!\int \rho^{\alpha-1}\,\Delta\pi^a\,{\rm d}\vx
&\;=\; -\tfrac{g^2}{2}\,\alpha\!\int \nabla(\rho^{\alpha-1})\!\cdot\!\nabla\pi^a\,{\rm d}\vx \\
&\;=\; -\tfrac{g^2}{2}\,\alpha(\alpha{-}1)\!\int \rho^{\alpha-2}\,\nabla\rho\!\cdot\!\nabla\pi^a\,{\rm d}\vx.
\end{align*}
Now use $\nabla\pi^a = \nabla(\rho\pi^b) = \pi^b\nabla\rho + \rho\nabla\pi^b$:
\begin{equation*}
= -\tfrac{g^2}{2}\,\alpha(\alpha{-}1)\!\int \rho^{\alpha-2}\big[\pi^b\|\nabla\rho\|^2 + \rho\,\nabla\rho\!\cdot\!\nabla\pi^b\big]\,{\rm d}\vx.
\end{equation*}
For the $I_b$ diffusion term:
\begin{align*}
-\tfrac{g^2}{2}\,(\alpha{-}1)\!\int \rho^\alpha\,\Delta\pi^b\,{\rm d}\vx
&\;=\; \tfrac{g^2}{2}\,(\alpha{-}1)\!\int \nabla(\rho^\alpha)\!\cdot\!\nabla\pi^b\,{\rm d}\vx \\
&\;=\; \tfrac{g^2}{2}\,\alpha(\alpha{-}1)\!\int \rho^{\alpha-1}\,\nabla\rho\!\cdot\!\nabla\pi^b\,{\rm d}\vx.
\end{align*}
Adding the two diffusive contributions, the $\nabla\rho\!\cdot\!\nabla\pi^b$ terms cancel:
\begin{equation*}
-\tfrac{g^2}{2}\,\alpha(\alpha{-}1)\!\int \rho^{\alpha-2}\,\rho\,\nabla\rho\!\cdot\!\nabla\pi^b\,{\rm d}\vx \;+\; \tfrac{g^2}{2}\,\alpha(\alpha{-}1)\!\int \rho^{\alpha-1}\,\nabla\rho\!\cdot\!\nabla\pi^b\,{\rm d}\vx \;=\; 0,
\end{equation*}
and we are left with
\begin{equation*}
\dot M_\alpha \;=\; -\tfrac{g(t)^2}{2}\,\alpha(\alpha{-}1)\!\int \rho^{\alpha-2}\,\|\nabla\rho\|^2\,\pi^b\,{\rm d}\vx.
\end{equation*}

\medskip\noindent\textbf{Step 5: Rewrite in terms of $\nabla\log\rho$.} \;
Since $\nabla\rho = \rho\,\nabla\log\rho$, we have $\rho^{\alpha-2}\|\nabla\rho\|^2 = \rho^\alpha\|\nabla\log\rho\|^2$, giving
\begin{equation*}
\dot M_\alpha \;=\; -\tfrac{g(t)^2}{2}\,\alpha(\alpha{-}1)\!\int \rho^\alpha\,\|\nabla\log\rho\|^2\,\pi^b\,{\rm d}\vx.
\end{equation*}
Since $\frac{{\rm d}}{{\rm d}t}D_\alpha = \frac{1}{\alpha-1}\cdot\frac{\dot M_\alpha}{M_\alpha}$, the factor $(\alpha-1)$ cancels, giving $\frac{{\rm d}}{{\rm d}t}D_\alpha = -\frac{g(t)^2}{2}\cdot\alpha\cdot\frac{\int\rho^\alpha\|\nabla\log\rho\|^2\pi^b\,{\rm d}\vx}{\int\rho^\alpha\pi^b\,{\rm d}\vx} = -\frac{g(t)^2}{2}\cdot\alpha\cdot\E_{\pi^{(\alpha)}}[\|\nabla\log\rho\|^2]$, which is Eq.~\eqref{eq:debruijn}.
\end{proof}

\subsection{Error induced by the surrogate}
\label{app:positive-policy-error}

We quantify the error introduced by replacing the reward-specific positive policy
$\pi_{i}^+$ with a common external positive policy $\pi^+$ when estimating the
reward-specific temporal weights. The main result shows that this approximation
is stable when $\pi^+$ is close to $\pi_{i}^+$ in KL divergence. Throughout this
subsection we fix a prompt $\vc$, suppress it from the notation, and take
$\alpha>1$ as in the main text.

\paragraph{Bounding the discriminability error.}
Assume the reward is bounded, $0\leq r_i(\vx_0)\leq R_i$ with $Z_i>0$, so that
$0\leq\rho_{i,t}(\vx_t)\leq M_i:=R_i/Z_i$ for every $t$, and define the policy
mismatch $\epsilon_i:=\mathrm{KL}\big(\pi^+\,\|\,\pi_{i}^+\big)$. Because
$\pi_t^+$ and $\pi_{i,t}^+$ are obtained by applying the same diffusion kernel,
the data-processing inequality gives
$\mathrm{KL}\big(\pi_t^+\,\|\,\pi_{i,t}^+\big)\leq\epsilon_i$, and Pinsker's
inequality then gives
$\mathrm{TV}\big(\pi_t^+,\pi_{i,t}^+\big)\leq\sqrt{\epsilon_i/2}$.
Since $f_{i,t}:=\rho_{i,t}^{\alpha-1}$ takes values in $[0,M_i^{\alpha-1}]$,
\begin{equation}
\Big|\E_{\pi_t^+}[f_{i,t}]-\E_{\pi_{i,t}^+}[f_{i,t}]\Big|
\;\leq\;
M_i^{\alpha-1}\,\mathrm{TV}\big(\pi_t^+,\pi_{i,t}^+\big)
\;\leq\;
M_i^{\alpha-1}\sqrt{\frac{\epsilon_i}{2}}
\;=:\;\eta_i.
\label{eq:expectation-error}
\end{equation}
Moreover $\E_{\pi_{i,t}^+}[f_{i,t}]=\exp\big((\alpha-1)D_{i,t}\big)\geq1$.
Write $\widetilde D_{i,t}:=\frac{1}{\alpha-1}\log\E_{\pi_t^+}[f_{i,t}]$ for the
surrogate log-moment, whose differences give the surrogate gain of
Eq.~\eqref{eq:surrogate-gain}. Whenever $\eta_i<1$,
Eq.~\eqref{eq:expectation-error} yields
\begin{equation}
\big|\widetilde D_{i,t}-D_{i,t}\big|
\;\leq\;
b_i:=\frac{-\log(1-\eta_i)}{\alpha-1}
\;=\;
\frac{M_i^{\alpha-1}}{\alpha-1}\sqrt{\frac{\epsilon_i}{2}}+O(\epsilon_i),
\label{eq:divergence-error}
\end{equation}
so the surrogate divergence is accurate to $O\big(\sqrt{\epsilon_i}\big)$.

\paragraph{Error in the per-timestep gain.}
Applying Eq.~\eqref{eq:divergence-error} at the two adjacent timesteps of the
gain $\Delta D_{i,t}=D_{i,t-1}-D_{i,t}$ and of its surrogate
$\widetilde{\Delta D}_{i,t}=\widetilde D_{i,t-1}-\widetilde D_{i,t}$ gives
\begin{equation}
\big|\widetilde{\Delta D}_{i,t}-\Delta D_{i,t}\big|
\;\leq\;
2b_i
\;\lesssim\;
\frac{\sqrt{2}\,M_i^{\alpha-1}}{\alpha-1}\sqrt{\epsilon_i},
\label{eq:gain-error-small}
\end{equation}
where the final comparison holds in the small-mismatch regime.

\paragraph{Error after normalizing the temporal profile.}
The allocation uses the mean-normalized gain of Eq.~\eqref{eq:gain-bar},
$\overline{\Delta D}_{i,t}=T\,\Delta D_{i,t}/G_i$ with
$G_i:=\sum_t\Delta D_{i,t}$, and its surrogate
$\widetilde{\overline{\Delta D}}_{i,t}:=T\,\widetilde{\Delta D}_{i,t}/\widetilde G_i$.
Because the gains telescope, $G_i=D_{i,0}-D_{i,T}$ and
$\widetilde G_i=\widetilde D_{i,0}-\widetilde D_{i,T}$, so
$|\widetilde G_i-G_i|\leq2b_i$. If $G_i>2b_i$, then for every $t$
\begin{equation}
\Big|\widetilde{\overline{\Delta D}}_{i,t}-\overline{\Delta D}_{i,t}\Big|
\leq
T\,\frac{|\widetilde{\Delta D}_{i,t}-\Delta D_{i,t}|}{\widetilde G_i}
+\overline{\Delta D}_{i,t}\,\frac{|G_i-\widetilde G_i|}{\widetilde G_i}
\leq
\frac{4T\,b_i}{G_i-2b_i}
=
O\!\left(\frac{T\,M_i^{\alpha-1}}{(\alpha-1)\,G_i}\sqrt{\epsilon_i}\right).
\label{eq:normalized-profile-error}
\end{equation}

\paragraph{Propagation to Sinkhorn weights.}
Finally, consider the entropy-regularized allocation, written with the cost
notation of Eq.~\eqref{eq:ot-form},
\begin{equation}
\bm{W}(\bm{C})
=
\argmin_{\bm{W}\in\mathcal U(\bm{\lambda},\bm{\nu})}
\;\sum_{i,t}\Big[\,C_{i,t}\,W_{i,t}+\tau\,W_{i,t}\log W_{i,t}\,\Big],
\qquad
C_{i,t}:=-\overline{\Delta D}_{i,t},
\label{eq:sinkhorn-opt}
\end{equation}
where $\mathcal U(\bm{\lambda},\bm{\nu})$ fixes the reward marginals to
$\bm{\lambda}$ and the timestep marginals to $\bm{\nu}$, $\tau>0$ is the
Sinkhorn temperature, and $B:=\sum_i\lambda_i=\sum_t\nu_t$ is the total
transported mass; at $\nu_t\equiv1/T$ and $\tau=1$, Eq.~\eqref{eq:sinkhorn-opt}
is exactly Eq.~\eqref{eq:ot-form}, with feasible set $\mathcal U(\bm{\lambda})$
of Eq.~\eqref{eq:polytope} and solution $\bm{W}^\star$. Let
$\bm{W}^\star:=\bm{W}(\bm{C})$ and
$\widetilde{\bm{W}}:=\bm{W}(\widetilde{\bm{C}})$, where
$\widetilde{\bm{C}}:=\big(-\widetilde{\overline{\Delta D}}_{i,t}\big)_{i,t}$ is
the cost built from the surrogate gains. The entropic term is
$\tau/B$-strongly convex with respect to the $\ell_1$ norm on measures of mass
$B$, so the optimality conditions for Eq.~\eqref{eq:sinkhorn-opt} at
$\bm{W}^\star$ and $\widetilde{\bm{W}}$ give
$\frac{\tau}{B}\|\widetilde{\bm{W}}-\bm{W}^\star\|_1^2
\leq\langle\bm{C}-\widetilde{\bm{C}},\,\widetilde{\bm{W}}-\bm{W}^\star\rangle$,
and H\"older's inequality then yields the stability estimate
$\|\widetilde{\bm{W}}-\bm{W}^\star\|_1\leq\frac{B}{\tau}\|\widetilde{\bm{C}}-\bm{C}\|_\infty$.
Since $\|\widetilde{\bm{C}}-\bm{C}\|_\infty$ is bounded by
Eq.~\eqref{eq:normalized-profile-error}, the end-to-end bound for sufficiently
small policy mismatch is
\begin{equation}
{
\|\widetilde{\bm{W}}-\bm{W}^\star\|_1
\;\leq\;
\frac{4BT}{\tau}\max_i\frac{b_i}{G_i-2b_i}
\;=\;
O\!\left(\frac{BT}{\tau}\max_i\frac{M_i^{\alpha-1}}{(\alpha-1)\,G_i}
\sqrt{\mathrm{KL}\big(\pi^+\|\pi_{i}^+\big)}\right).
}
\label{eq:weight-final-asymptotic}
\end{equation}
Eq.~\eqref{eq:weight-final-asymptotic} gives a direct interpretation of the
approximation: a shared positive policy produces similar temporal weights
whenever it is close to the reward-specific positive policy, with the square
root coming from Pinsker's inequality. Rewards with larger total
discriminability $G_i$ are less sensitive to the approximation, while a smaller
Sinkhorn temperature $\tau$ makes the final allocation more sensitive to errors
in the estimated temporal profile.

\section{Gain estimation: algorithm, \texorpdfstring{$\alpha$}{alpha}-sensitivity, and curve details}
\label{app:estimation}

\subsection{Algorithm}
\label{app:estimation-alg}

Alg.~\ref{alg:estimation-full} writes out Alg.~\ref{alg:estimation} in full, separated into its three phases, and implements the estimator of Sec.~\ref{sec:estimation}. Phase 1 forward-noises clean $\pi^+$ samples to obtain $\vx_t\sim\pi_t^+$; Phase 2 rolls out under $\pi^{\mathrm{old}}$ to estimate the unnormalized conditional reward at each $\vx_t$; Phase 3 differences the resulting log-moments, at which point the marginal $\E^{\mathrm{old}}[r_i\mid\vc]$ drops out and never has to be formed, leaving the surrogate gain of Eq.~\eqref{eq:surrogate-gain}.

\begin{algorithm}[!htbp]
\caption{Estimation of per-step R\'enyi discriminability gain proxies $\widehat{\widetilde{\Delta D}_{i,t}}$ from $\pi^+$ (the expanded form of Alg.~\ref{alg:estimation})}
\label{alg:estimation-full}
\begin{algorithmic}[1]
\Require Prompt $\vc$, current policy $\pi^{\mathrm{old}}$, external high-reward generator approximating $\pi^+(\vx_0\mid\vc)$, rewards $r_1,\dots,r_m$, R\'enyi order $\alpha>0$ with $\alpha\neq1$, number of $\pi^+$ samples $N$, number of rollouts $K$, noise schedule $\{\sigma_t\}_{t=0}^T$
\Ensure Per-step gain proxies $\widehat{\widetilde{\Delta D}_{i,t}}$ for each reward $i$ and timestep $t$
\State \textbf{Phase 1: Sample $\vx_t \sim \pi_t^+$ by forward-noising clean $\pi^+$ samples}
\For{$n = 1, \dots, N$}
    \State Draw clean sample $\widetilde{\vx}_0^{(n)} \sim \pi^+(\cdot\mid\vc)$ from the stronger external generator
    \For{each timestep $t \in \{0, 1, \dots, T\}$}
        \State Sample $\vepsilon_t^{(n)} \sim \mathcal N(0, I)$ and set $\vx_t^{(n)} \gets (1-\sigma_t)\,\widetilde{\vx}_0^{(n)} + \sigma_t\,\vepsilon_t^{(n)}$ \Comment{$\vx_t^{(n)}\sim\pi_t^+$}
    \EndFor
\EndFor
\State \textbf{Phase 2: Rollouts under $\pi^{\mathrm{old}}$ from each $\vx_t^{(n)}$ to estimate the conditional reward}
\For{$n = 1, \dots, N$}
    \For{each timestep $t \in \{0, 1, \dots, T\}$}
        \If{$t = 0$} \Comment{$\widetilde{\vx}_0^{(n)}$ is already clean: no rollout, single deterministic evaluation}
            \State $\widetilde\vx_0^{(n,0,1)} \gets \widetilde{\vx}_0^{(n)}$; evaluate $r_i(\widetilde\vx_0^{(n,0,1)}, \vc)$ for all $i$
        \Else
            \For{$k = 1, \dots, K$}
                \State Run independent rollout from $\vx_t^{(n)}$ to completion under $\pi^{\mathrm{old}}$: $\vx_t^{(n)} \to \vx_0^{(n,t,k)}$
                \State Evaluate $r_i(\vx_0^{(n,t,k)}, \vc)$ for all rewards $i$
            \EndFor
        \EndIf
    \EndFor
\EndFor
\State \textbf{Phase 3: Per-step gains, with denominator cancellation}
\For{each reward $i$}
    \For{each timestep $t \in \{0,1,\dots,T\}$}
        \For{$n = 1, \dots, N$}
            \State $\hat\mu_{i,t}^{(n)} \gets \frac{1}{K_t}\sum_{k=1}^{K_t} r_i(\vx_0^{(n,t,k)}, \vc)$, with $K_t = 1$ if $t = 0$ else $K$ \Comment{Conditional reward at $\vx_t^{(n)}$}
        \EndFor
        \State $\widehat{\widetilde D_{i,t}} \gets \frac{1}{\alpha-1}\log\!\Big(\max\{10^{-10},\frac{1}{N}\sum_{n=1}^N \big[\hat\mu_{i,t}^{(n)}\big]^{\alpha-1}\}\Big)$ \Comment{Surrogate log-moment, with a numerical floor}
    \EndFor
    \For{$t = 1, \dots, T$}
        \State $\widehat{\widetilde{\Delta D}_{i,t}} \gets \big[\widehat{\widetilde D_{i,t-1}} - \widehat{\widetilde D_{i,t}}\big]_+$ \Comment{Eq.~\eqref{eq:surrogate-gain}, clipped as in Sec.~\ref{sec:sinkhorn}}
    \EndFor
\EndFor
\end{algorithmic}
\end{algorithm}

\subsection{\texorpdfstring{$\alpha$}{alpha}-sensitivity of the R\'enyi estimator}
\label{app:alpha-sweep}

To analyze the influence of $\alpha$, we sweep $\alpha\in\{0.25,\,0.5,\,0.75,\,1.5,\,2,\,4,\,8\}$ and plot per-step gain, mean-normalized gain, and the Sinkhorn matrix they project to in Fig.~\ref{fig:renyi-multialpha}. The sweep covers the five rewards estimated jointly on the shared $T{=}10$ grid. For $0<\alpha<1$ the exponent $(\alpha-1)$ is negative, so the log-moment is dominated by the lower tail of $\hat\mu_{i,t}$ rather than the upper one. Of the two guarantees we rely on, only one needs $\alpha>1$: Theorem~\ref{thm:debruijn} gives $\Delta D_{i,t}\ge0$ for every $\alpha>0$ with $\alpha\neq1$, whereas $D_\alpha$ upper-bounds the KL only above $1$. Below $1$ the gain therefore keeps its monotonicity but loses its link to the original objective, and we report these orders only to show what the estimator does outside its intended range. Peak locations are stable through the middle of the range and move only at its ends: \textsc{Aesthetic} peaks at the clean end for $\alpha\le 0.75$ but near the noisy end for $\alpha\ge 1.5$, and at $\alpha{=}8$ \textsc{HPSv2} and \textsc{PickScore} lose their own peaks closer to clean and collapse onto \textsc{ImageReward}'s at $t{=}9$. Raising $\alpha$ also localizes the gains: \textsc{ImageReward}'s normalized peak grows from $1.61\times$ its own mean at $\alpha{=}0.25$ to $2.12\times$ at $\alpha{=}2$ and $3.25\times$ at $\alpha{=}8$. The projected weights follow, and their excursions from uniform are mildest in the intermediate range, spanning $0.29$ to $1.60$ times $1/(mT)$ at $\alpha{=}2$ against $0.24$ to $1.89$ at $\alpha{=}0.25$ and $0.41$ to $2.24$ at $\alpha{=}8$, which is why we default to $\alpha=2$ in the main text.

\begin{figure}[!htbp]
\centering
\includegraphics[width=\linewidth]{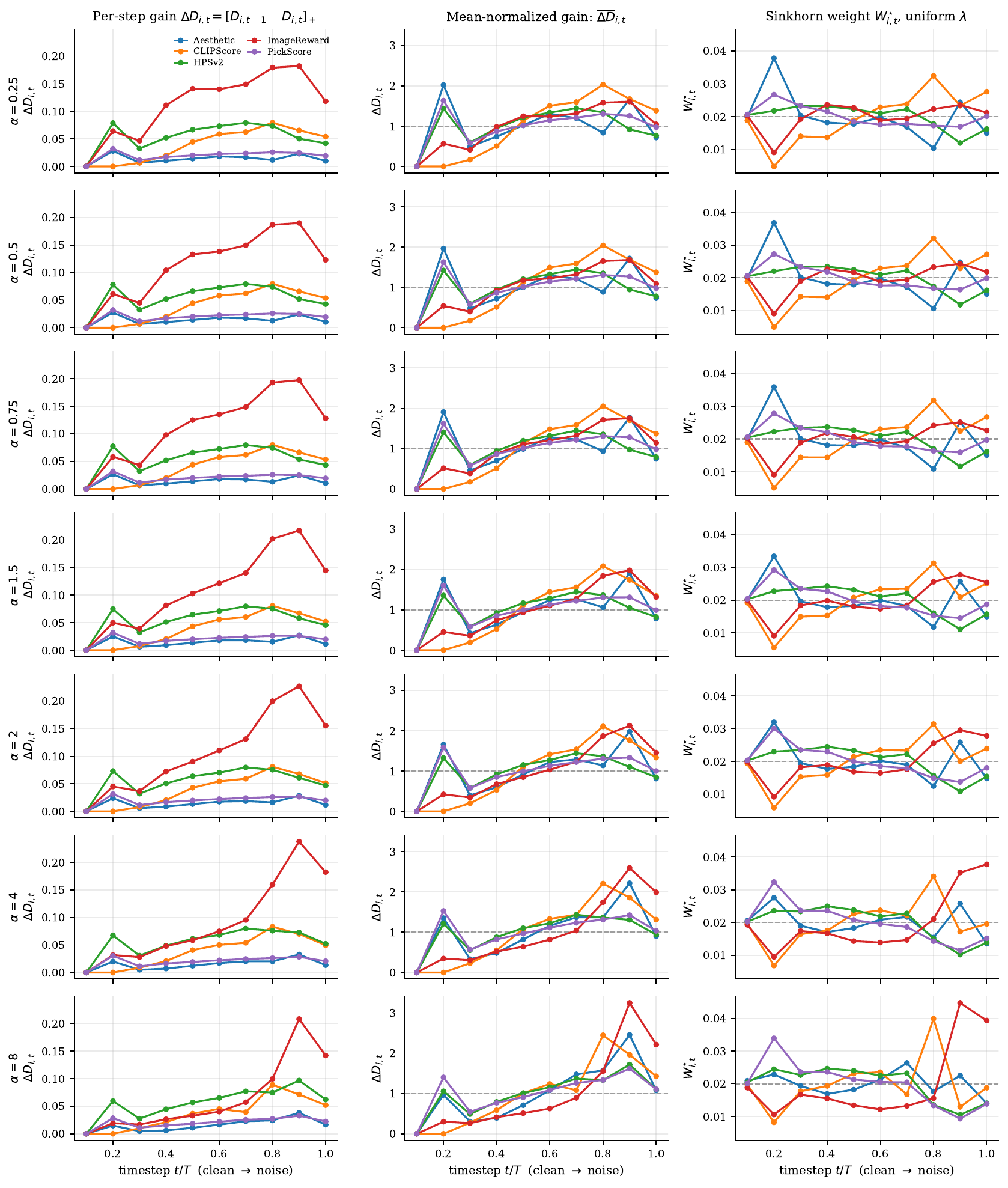}
\caption{$\alpha$-sensitivity of the $\pi^+$-sampling R\'enyi estimator on the SD3.5-Medium run of Sec.~\ref{sec:empirical-omega}. \textbf{Rows:} $\alpha\in\{0.25,\,0.5,\,0.75,\,1.5,\,2,\,4,\,8\}$. \textbf{Columns:} the same three quantities as Fig.~\ref{fig:renyi-omega}, per-step gain $\Delta D_{i,t}$ (left), mean-normalized gain $\overline{\Delta D}_{i,t}$ (Eq.~\eqref{eq:gain-bar}, center), and the Sinkhorn matrix $W_{i,t}^\star$ (Eq.~\eqref{eq:sinkhorn-form}) at uniform $\lambda_i{=}1/m$ (right).}
\label{fig:renyi-multialpha}
\end{figure}

\subsection{\texorpdfstring{$\pi^+$}{pi+} vs.\ \texorpdfstring{$\pi^{\mathrm{old}}$}{pi\^{}old} self-rollout gain curves}
\label{app:psampling-ablation}

In this section, we validate the design of $\pi^+$ sampling. To this end, we run the identical estimator and Sinkhorn recipe with $\pi^{\mathrm{old}}$ self-rollouts in place of $\pi^+$ samples, replacing the exponent $\alpha-1$ by $\alpha$ so that both runs target the same $D_{i,t}$, since $\E_{\pi_t^{\mathrm{old}}}[\rho_{i,t}(\vx_t)^\alpha]=\E_{\pi_{i,t}^+}[\rho_{i,t}(\vx_t)^{\alpha-1}]$ (Sec.~\ref{sec:estimation}), holding the training stage-OCR prompt set, the four training rewards (\textsc{ClipScore}, \textsc{HPSv2}, \textsc{OCR}, \textsc{PickScore}), $\alpha{=}2$ and $T{=}25$ fixed, so the two runs differ only in the sampling distribution. As shown in Fig.~\ref{fig:renyi-omega-combined}, the $\pi^+$-sampled $\Delta D$ curves (top) are smooth and single-peaked, with the peak location differing by reward, while the $\pi^{\mathrm{old}}$-sampled curves are jagged and non-monotonic step-to-step, with per-step gain repeatedly spiking and collapsing back to the zero clip. 
The leftmost column makes the same point quantitatively. At the clean end the $\pi^{\mathrm{old}}$ estimate gives $D_{i,t}$ of $0.016$, $0.020$, and $0.001$ for \textsc{ClipScore}, \textsc{HPSv2}, and \textsc{PickScore}, against $0.245$, $0.501$, and $0.209$ under $\pi^+$: a $15$--$200\times$ collapse, and precisely the $\rho_{i,t}\approx1$, $D_{i,t}\approx0$ degeneracy anticipated in Sec.~\ref{sec:estimation}. \textsc{OCR} is the one reward that retains appreciable discriminability, but its $\Delta D$ curve is unstable in $t$.
The Sinkhorn matrices built from these noisy curves (right column) are correspondingly noisy themselves, oscillating step-to-step rather than smoothly redistributing weight toward each reward's high-gain region. In short: $\pi^+$ sampling is what makes the estimated curve usable as a Sinkhorn kernel input.

\paragraph{The jaggedness is not an artifact of the budget.}
A natural objection is that the $\pi^{\mathrm{old}}$ curves are jagged only because the estimation budget is small. The bottom row of Fig.~\ref{fig:renyi-omega-combined} rules this out: it recomputes the same $\pi^{\mathrm{old}}$ run at $N{=}32$ and $K{=}64$, a $4\times$ increase in both the samples per prompt and the rollouts per $\vx_t$ and hence $16\times$ the rollout cost, so the two $\pi^{\mathrm{old}}$ rows differ in $(N,K)$ alone. 
The $\Delta D$ curves stay just as jagged, so the shape the Sinkhorn kernel would consume is not yet stable. The extra budget therefore almost does not improve smoothness. $\pi^+$ sampling instead reaches a usable curve at $1/16$ of the rollout cost of this row.

\begin{figure}[!htbp]
\centering
\includegraphics[width=\linewidth]{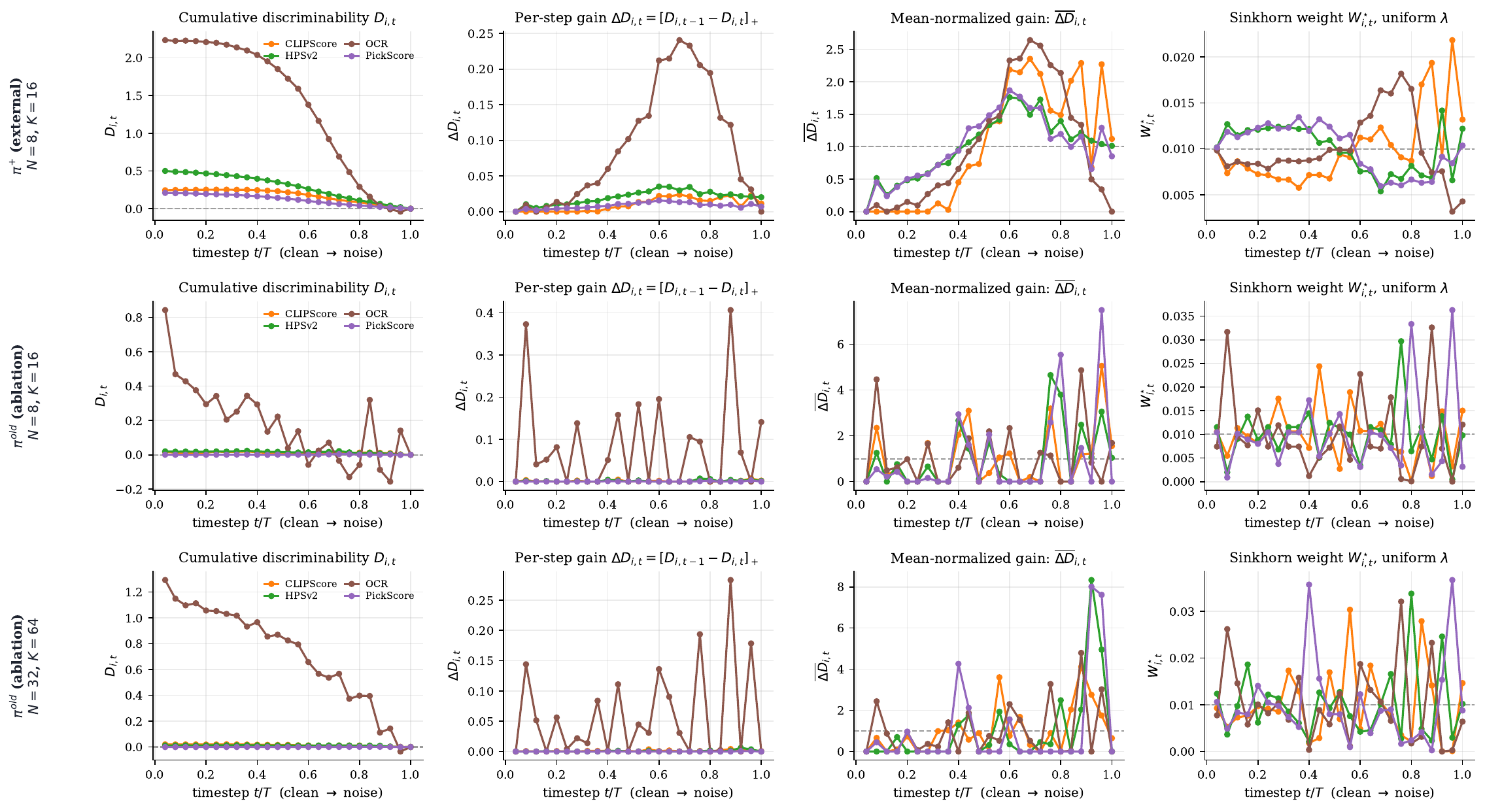}
\caption{Cumulative discriminability $D_{i,t}$, per-step gain $\Delta D_{i,t}$, mean-normalized gain $\overline{\Delta D}_{i,t}$, and Sinkhorn matrix $W_{i,t}^\star$ at uniform $\lambda_i{=}1/m$, left to right in the order they are derived, under the two sampling distributions. The leftmost column is the estimator's own $D_{i,t}$, shifted by the single $t$-independent constant that sets $D_{i,T}{=}0$ at pure noise, where $\pi_{i,T}^+{=}\pi_T^{\mathrm{old}}$; it is not a cumulative sum of the clipped gains beside it. Under $\pi^{\mathrm{old}}$ almost all of the estimated discriminability sits on \textsc{OCR}, the other three rewards coming out nearly indiscriminable at every step, and the \textsc{OCR} curve is not even monotone in $t$: it dips below zero, which Theorem~\ref{thm:debruijn} rules out in population. \textbf{Top:} $\pi^+$ samples from \textsc{GPT Image 1.5}, at the $N{=}8$, $K{=}16$ budget used throughout the paper. \textbf{Middle:} $\pi^{\mathrm{old}}$ self-rollouts at that same budget. \textbf{Bottom:} the same $\pi^{\mathrm{old}}$ run at the full $N{=}32$, $K{=}64$ estimation budget, $16\times$ the rollout cost. All three rows are matched estimator runs on the training stage-OCR curve: same four training rewards (\textsc{ClipScore}, \textsc{HPSv2}, \textsc{OCR}, \textsc{PickScore}), same prompt set, same $\alpha{=}2$ and $T{=}25$; the top two differ only in the sampling distribution, the bottom two only in $(N,K)$. Both $\pi^{\mathrm{old}}$ rows are visibly noisier at every step than the $\pi^+$ row, consistent with the higher-variance $\pi^{\mathrm{old}}$-side estimator, and the $16\times$ budget does not smooth them.}
\label{fig:renyi-omega-combined}
\end{figure}

\subsection{Sensitivity to the \texorpdfstring{$\pi^+$}{pi+} surrogate generator}
\label{app:surrogate-sensitivity}

The curves above approximate $\pi^+$ by a single external generator, \textsc{GPT Image 1.5}. To test how much the resulting weights depend on the specific $\pi^+$, we regenerate the training stage-OCR $\pi^+$ image set with an independent generator, \textsc{Nano Banana Pro} (\textsc{Gemini 3 Pro Image}), same prompts and $N{=}8$ images per prompt, and rerun the identical estimator. Tab.~\ref{tab:surrogate-sensitivity} compares the two per-step curves reward-by-reward. The curves are highly correlated (Pearson $r\in[0.84,0.92]$, Spearman $\rho\in[0.82,0.92]$) and peak within $0$--$2$ steps of each other. Fig.~\ref{fig:renyi-omega-surrogate} shows the two curves side by side: the per-step gains, mean-normalized gains, and the Sinkhorn matrices they project to are visually near-interchangeable. The gain curve, and hence the Sinkhorn weights built from it, is therefore largely a property of the reward and the trajectory, not of which strong external generator stands in for $\pi^+$.

\begin{table}[!htbp]
\centering
\small
\caption{Surrogate sensitivity of the training stage-OCR gain curve: \textsc{GPT Image 1.5} vs.\ \textsc{Nano Banana Pro} as the $\pi^+$ generator, all other estimator settings identical. $r$ and $\rho$ are Pearson and Spearman correlations between the two per-step curves over the $T{=}25$ steps. ``Peak $t$'' gives the argmax timestep of $\Delta D_{i,t}$ in the paper's convention ($t{=}1$ clean, $t{=}T$ pure noise); ``total gain'' is the ratio of $\sum_t \Delta D_{i,t}$ under \textsc{Nano Banana Pro} to that under \textsc{GPT Image 1.5}.}
\label{tab:surrogate-sensitivity}
\begin{tabular}{@{}lcccc@{}}
\toprule
Reward & Pearson $r$ & Spearman $\rho$ & Peak $t$ & Total gain \\
\midrule
\textsc{ClipScore} & $0.879$ & $0.912$ & $17 \to 19$ & $0.94\times$ \\
\textsc{HPSv2} & $0.847$ & $0.829$ & $15 \to 16$ & $0.81\times$ \\
\textsc{OCR} & $0.862$ & $0.844$ & $17 \to 17$ & $0.92\times$ \\
\textsc{PickScore} & $0.919$ & $0.915$ & $15 \to 16$ & $0.84\times$ \\
\bottomrule
\end{tabular}
\end{table}

\begin{figure}[!htbp]
\centering
\includegraphics[width=\linewidth]{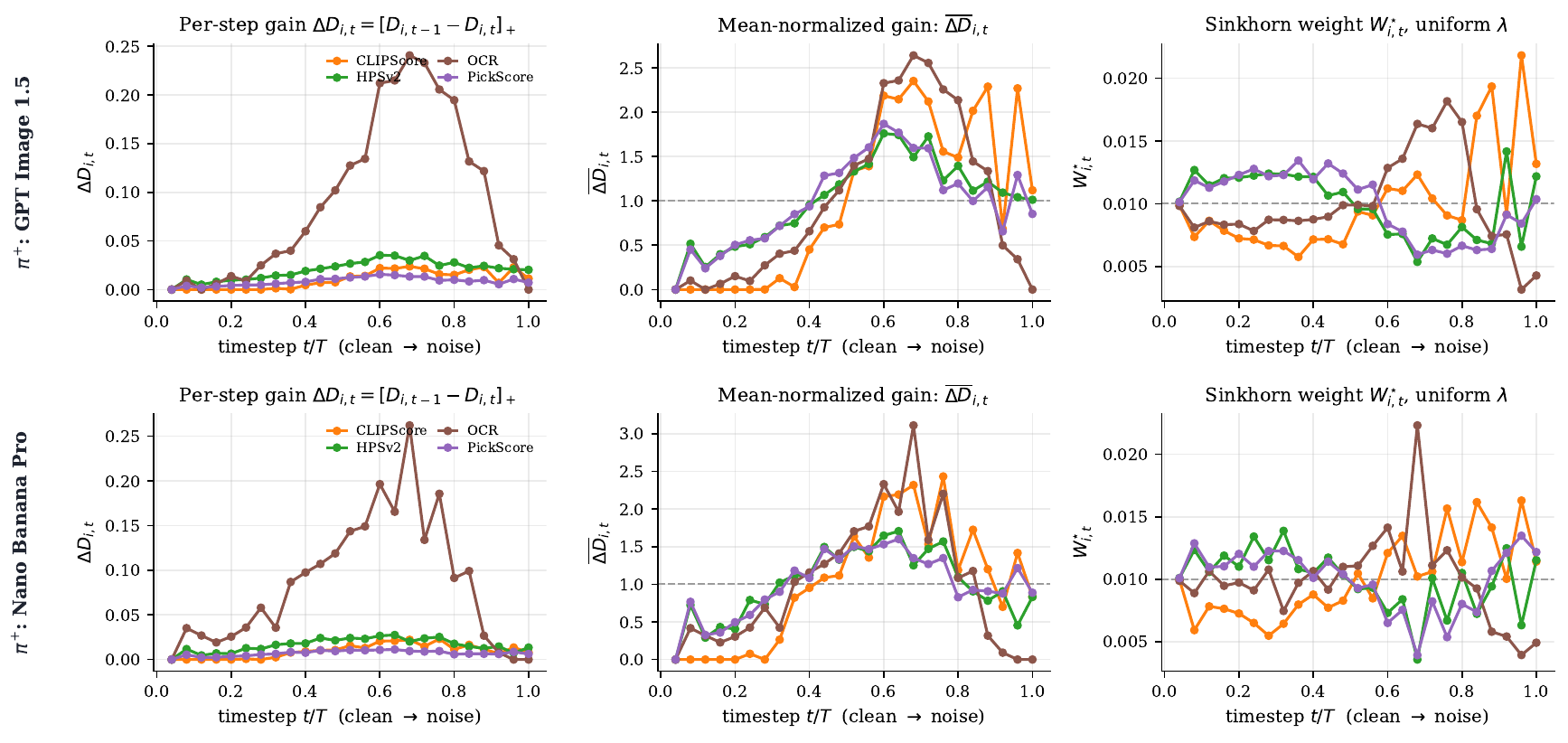}
\caption{Per-step gain $\Delta D_{i,t}$, mean-normalized gain $\overline{\Delta D}_{i,t}$, and Sinkhorn matrix $W_{i,t}^\star$ at uniform $\lambda_i{=}1/m$, for the training stage-OCR curve under the two $\pi^+$ surrogate generators. \textbf{Top:} \textsc{GPT Image 1.5}. \textbf{Bottom:} \textsc{Nano Banana Pro}. The two rows produce near-identical reward-specific shapes and peak orderings (quantified in Tab.~\ref{tab:surrogate-sensitivity}).}
\label{fig:renyi-omega-surrogate}
\end{figure}

\subsection{Closing the divergence to \texorpdfstring{$\pi^+$}{pi+} over a training stage}
\label{app:curve-drift}

In this section, we measure how much of the divergence the curve is built from survives training against it. We rerun the identical training stage-OCR estimator with $\pi^{\mathrm{old}}$ set to checkpoint-$180$ of the $\bm{\lambda}{=}(1,1,1,1)$ \ourmethod{} stage-OCR run, keeping the \textsc{GPT Image 1.5} $\pi^+$ samples and every other estimator setting fixed, so the policy the gains are measured at is the only thing that changes.

By Eq.~\eqref{eq:telescope}, the total gain $\sum_t \Delta D_{i,t}$ is the data-level divergence $D_\alpha\!\big(\pi_{i,0}^+\,\|\,\pi_0^{\mathrm{old}}\big)$ still separating the policy from the reward-induced positive distribution, so it reads directly as how much of reward $i$'s gap is left to close. Tab.~\ref{tab:curve-drift} shows that by checkpoint-$180$ almost none of it is: the total gain falls to $0.14\times$ its base-policy value for \textsc{ClipScore} and to $0.01$--$0.11\times$ for the other three rewards, and to $0.03\times$ summed over all four. Fig.~\ref{fig:renyi-omega-drift} shows the same collapse in $D_{i,t}$ itself: on a shared vertical axis, the trained policy's curves sit near zero across the whole trajectory. Because R\'enyi divergence is nondecreasing in $\alpha$, the $\alpha{=}2$ quantity we measure upper-bounds $\mathrm{KL}\!\big(\pi_{i,0}^+\,\|\,\pi_0^{\mathrm{old}}\big)$, so the KL gap to the positive distribution is closed at least as far. Training under the reweighted objective therefore closes almost all of the gap the curve is built from, and it does so for all four rewards simultaneously rather than closing one reward's gap at another's expense.

\begin{table}[!htbp]
\centering
\small
\caption{Total discriminability gain $\sum_t \Delta D_{i,t}$ of the training stage-OCR curve with $\pi^{\mathrm{old}}$ at the untrained base policy vs.\ at the trained checkpoint-$180$, with the \textsc{GPT Image 1.5} $\pi^+$ samples and all other estimator settings identical. By Eq.~\eqref{eq:telescope} each entry is the data-level divergence remaining between that policy and the reward's positive distribution.}
\label{tab:curve-drift}
\begin{tabular}{@{}lccc@{}}
\toprule
Reward & Base policy & ckpt-$180$ & Remaining \\
\midrule
\textsc{ClipScore} & $0.252$ & $0.036$ & $0.14\times$ \\
\textsc{HPSv2} & $0.501$ & $0.024$ & $0.05\times$ \\
\textsc{OCR} & $2.278$ & $0.028$ & $0.01\times$ \\
\textsc{PickScore} & $0.209$ & $0.024$ & $0.11\times$ \\
\midrule
All four & $3.241$ & $0.113$ & $0.03\times$ \\
\bottomrule
\end{tabular}
\end{table}

\begin{figure}[!htbp]
\centering
\includegraphics[width=\linewidth]{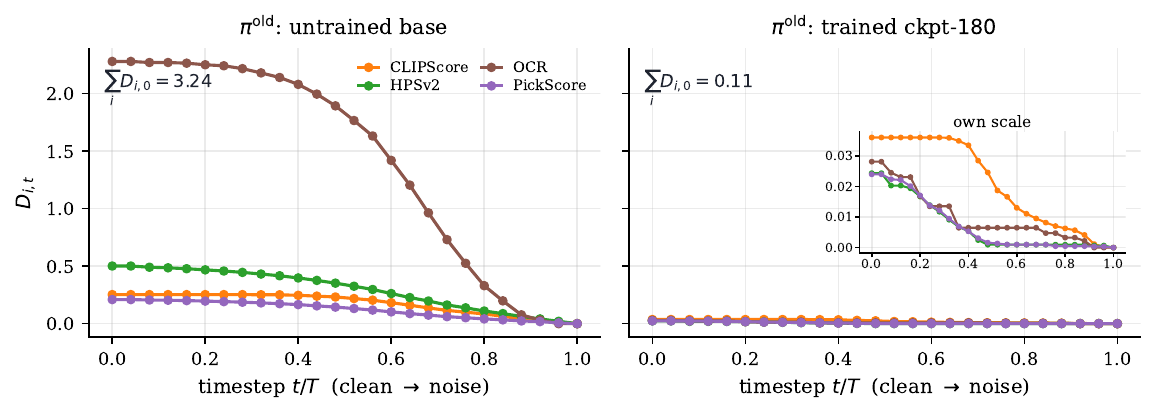}
\caption{Cumulative discriminability $D_{i,t}$ of the training stage-OCR curve at the two policies, with $\pi^+$ set to \textsc{GPT Image 1.5} and all other estimator settings identical. \textbf{Left:} $\pi^{\mathrm{old}}$ at the untrained base. \textbf{Right:} $\pi^{\mathrm{old}}$ at checkpoint-$180$ of the $\bm{\lambda}{=}(1,1,1,1)$ \ourmethod{} stage-OCR run, drawn on the same vertical axis; the inset repeats it at its own scale. Each curve telescopes the per-step gains up from $D_{i,T}=0$ (Eq.~\eqref{eq:telescope}), so its clean endpoint $D_{i,0}$ is the full data-level divergence between that policy and reward $i$'s positive distribution. Training leaves $D_{i,t}$ smaller at every timestep, for every reward.}
\label{fig:renyi-omega-drift}
\end{figure}

\section{Sinkhorn projection: solver and alternatives}
\label{app:sinkhorn-solver}

\subsection{The log-domain iteration}
\label{app:sinkhorn-alg}

Alg.~\ref{alg:sinkhorn} is the solver of Eq.~\eqref{eq:sinkhorn-problem}. It normalizes the gains per reward, exponentiates them into the kernel, and alternates the two marginal updates of Eq.~\eqref{eq:sinkhorn-iter} in the log domain, which keeps the scaling vectors stable when a kernel entry is tiny. Convergence is linear and in practice takes fewer than $100$ iterations at $m\leq4$, $T\leq25$; the cost is negligible next to a single training step, and the matrix is computed once per $\bm{\lambda}$ before training.

\begin{algorithm}[!htbp]
\caption{Sinkhorn construction of the reward-by-timestep weight matrix $\bm{W}$}
\label{alg:sinkhorn}
\begin{algorithmic}[1]
\Require Per-step gains $\{\Delta D_{i,t}\}$ (Alg.~\ref{alg:estimation}, Sec.~\ref{sec:estimation}), inter-reward budget $\bm{\lambda}\in\Delta^{m-1}$, tolerance $\varepsilon_{\mathrm{tol}}$, max iterations $L$
\Ensure weight matrix $\bm{W}\in\mathcal U(\bm{\lambda})$, used in the loss as $T\!\cdot\!W_{i,t}$
\For{each reward $i$}
    \State $\overline{\Delta D}_{i,t} \gets \Delta D_{i,t} \,/\, \big(\tfrac1T\sum_s \Delta D_{i,s}\big)$ \Comment{Eq.~\eqref{eq:gain-bar}: per-reward mean-normalize}
    \State $\log K_{i,t} \gets \overline{\Delta D}_{i,t}$ \Comment{Eq.~\eqref{eq:kernel}; $\log K_{i,t}\gets0$ if reward $i$ has no gain estimate}
\EndFor
\State $\log \va \gets \vzero_m$, \; $\log \vb \gets \vzero_T$
\For{$\ell = 1,\dots,L$}
    \State $\log a_i \gets \log\lambda_i - \operatorname{logsumexp}_t\!\big(\log K_{i,t} + \log b_t\big)$ \Comment{row marginals $\to\lambda_i$}
    \State $\log b_t \gets -\log T - \operatorname{logsumexp}_i\!\big(\log K_{i,t} + \log a_i\big)$ \Comment{column marginals $\to 1/T$}
    \State $W_{i,t} \gets \exp\!\big(\log a_i + \log K_{i,t} + \log b_t\big)$
    \State \textbf{break if} $\max_i\big|\sum_t W_{i,t} - \lambda_i\big| < \varepsilon_{\mathrm{tol}}$ \textbf{and} $\max_t\big|\sum_i W_{i,t} - \tfrac1T\big| < \varepsilon_{\mathrm{tol}}$
\EndFor
\State \Return $\bm{W}$
\end{algorithmic}
\end{algorithm}

\subsection{Boundary cases}
\label{app:sinkhorn-cases}

Our construction sits between the alternatives one might have used instead, and each of them is instructive.

\paragraph{A flat kernel: the static baseline.} If no gain curve has any shape in $t$, then $\overline{\Delta D}_{i,t}\equiv1$ for every reward, so $K_{i,t}\equiv e$ is rank one and the entropic optimum over $\mathcal U(\bm{\lambda})$ is $W_{i,t}^\star = \lambda_i/T$, which is the standard static convex combination.

\paragraph{No entropic term: hard assignment.} Dropping the entropy from Eq.~\eqref{eq:ot-form} leaves an unregularized transport LP, whose solutions are vertices of $\mathcal U(\bm{\lambda})$: each timestep's budget goes essentially to the single reward with the largest normalized gain there. Committing that hard to argmax gains estimated from finite samples is brittle, so we keep the entropic term, which is what makes the solution the smooth, strictly positive matrix of Eq.~\eqref{eq:sinkhorn-form}.

\paragraph{An uncurved reward.} Because Eq.~\eqref{eq:gain-bar} normalizes each reward to mean $1$, the kernel needs no scale or sharpness parameter, and Sec.~\ref{sec:empirical-omega}'s shapes give smooth, non-degenerate matrices. If a reward has no estimated curve at all, we set $K_{i,t}=1$ for every $t$: it still receives its full row budget $\lambda_i$, spread uniformly up to the shared $b_t$. 

\section{Per-\texorpdfstring{$\bm{\lambda}$}{lambda} results}
\label{app:preagg-geneval}

Tabs.~\ref{tab:trainreward} and~\ref{tab:panels} pool the five budgets; this appendix resolves both of them per $\bm{\lambda}$, static vs.\ \ourmethod{} at matched $\bm{\lambda}$, every entry the mean over the $3$ seed replicates with its standard deviation as a subscript. Each budget also carries a Pareto verdict: \oursshort{}~$\succ$~static if the \ourmethod{} mean is at least as high as the static mean on \emph{every} reward in that table and strictly higher on at least one, static~$\succ$~\oursshort{} if static (weakly) dominates \ourmethod{} in the same sense, \emph{mixed} otherwise. Every cell scores the run's final checkpoint.

\subsection{Training rewards}

Tabs.~\ref{tab:perlam-ocr} and~\ref{tab:preagg-geneval} give the per-$\bm{\lambda}$ training rewards behind Tab.~\ref{tab:trainreward}. In training stage-OCR (Tab.~\ref{tab:perlam-ocr}) \ourmethod{} Pareto-dominates static at $3$ of the $5$ budgets, namely $(1,1,1,1)$, $(1,1,1,2)$, and $(1,1,2,1)$, and is dominated at none; the aggregate is higher at $4$ of $5$. The two mixed budgets fail the test on the same coordinate: $(1,2,1,1)$ raises \textsc{ClipScore}, \textsc{HPSv2}, and \textsc{PickScore} while conceding \textsc{OCR} by $0.039$, and $(2,1,1,1)$ raises \textsc{ClipScore} and \textsc{PickScore}, ties \textsc{HPSv2}, and concedes \textsc{OCR} by $0.026$. This is the per-budget form of the pattern in Tab.~\ref{tab:trainreward}, where the target reward carries by far the largest seed variance of the four.

Training stage-GenEval (Tab.~\ref{tab:preagg-geneval}) is a tie budget by budget as well as on average. \ourmethod{} dominates at no $\bm{\lambda}$, static dominates at none, and all $5$ budgets are mixed; at every one of them the aggregate shifts by less than the seed standard deviation of either method there. The target \textsc{GenEval} reward moves in both directions across budgets, up at $3$ of $5$ and down at $(1,1,1,2)$ and $(1,2,1,1)$, so no budget in this setting resolves the two methods.

\subsection{Held-out judges}

Tabs.~\ref{tab:perlam-panel-ocr} and~\ref{tab:perlam-panel-geneval} resolve Tab.~\ref{tab:panels} the same way. In training stage-OCR (Tab.~\ref{tab:perlam-panel-ocr}) \ourmethod{} Pareto-dominates static across all four judges at $3$ of the $5$ budgets, $(1,1,2,1)$, $(1,2,1,1)$, and $(2,1,1,1)$, and is dominated at none. \textsc{ImageReward}, \textsc{Aesthetic}, and \textsc{HPSv3} improve at every budget without exception; the two mixed budgets, $(1,1,1,1)$ and $(1,1,1,2)$, concede only \textsc{UnifiedReward-2}, by $0.025$ and $0.017$. The generalization gain of Sec.~\ref{sec:exp-heldout} is therefore not carried by a subset of the budgets.

Training stage-GenEval (Tab.~\ref{tab:perlam-panel-geneval}) is the setting whose training rewards tie, and its judge scores favor \ourmethod{} at a majority of budgets: \ourmethod{} dominates at $(1,1,1,2)$, $(1,1,2,1)$, and $(1,2,1,1)$, static dominates at none, and $(1,1,1,1)$ and $(2,1,1,1)$ are mixed. The seed standard deviations here are up to an order of magnitude wider than in training stage-OCR, wide enough that the individual budgets are not resolved by three replicates even where the means separate; the pooled comparison in Tab.~\ref{tab:panels}, which averages the same runs, still favors \ourmethod{} on all four judges.

\begin{table}[!htbp]
\centering
\small
\setlength{\tabcolsep}{4pt}
\caption{Training stage-OCR, per-$\bm{\lambda}$ training rewards on the held-out dataset at each run's final checkpoint, static vs.\ \ourmethod{}. Every entry is the mean over the $3$ seed replicates with its standard deviation as a subscript, Aggregate $r$ is the $\bm{\lambda}$-weighted sum as in Tab.~\ref{tab:trainreward}, and bold marks the better method per cell, neither when the two agree at the reported precision. Pareto is the verdict over the four training rewards.}
\label{tab:perlam-ocr}
\begin{tabular}{@{}llcccccl@{}}
\toprule
$\bm{\lambda}$ & Method & Aggregate $r$ & \textsc{ClipScore} & \textsc{HPSv2} & \textsc{PickScore} & \textsc{OCR} & Pareto \\
\midrule
\multirow{2}{*}{$(1,1,1,1)$} & static & 2.387$_{\pm0.058}$ & 0.315$_{\pm0.003}$ & 0.288$_{\pm0.007}$ & 0.890$_{\pm0.002}$ & 0.894$_{\pm0.051}$ & \multirow{2}{*}{\oursshort{} $\succ$ static} \\
 & \oursshort{} & \textbf{2.435}$_{\pm0.058}$ & \textbf{0.321}$_{\pm0.001}$ & \textbf{0.306}$_{\pm0.006}$ & \textbf{0.900}$_{\pm0.004}$ & \textbf{0.908}$_{\pm0.055}$ &  \\
\addlinespace[2pt]
\multirow{2}{*}{$(1,1,1,2)$} & static & 3.251$_{\pm0.030}$ & 0.308$_{\pm0.004}$ & 0.255$_{\pm0.022}$ & 0.869$_{\pm0.008}$ & 0.910$_{\pm0.001}$ & \multirow{2}{*}{\oursshort{} $\succ$ static} \\
 & \oursshort{} & \textbf{3.383}$_{\pm0.031}$ & \textbf{0.319}$_{\pm0.004}$ & \textbf{0.290}$_{\pm0.022}$ & \textbf{0.887}$_{\pm0.011}$ & \textbf{0.944}$_{\pm0.006}$ &  \\
\addlinespace[2pt]
\multirow{2}{*}{$(1,1,2,1)$} & static & 3.259$_{\pm0.039}$ & 0.313$_{\pm0.003}$ & 0.296$_{\pm0.004}$ & 0.898$_{\pm0.002}$ & 0.854$_{\pm0.035}$ & \multirow{2}{*}{\oursshort{} $\succ$ static} \\
 & \oursshort{} & \textbf{3.338}$_{\pm0.043}$ & \textbf{0.320}$_{\pm0.004}$ & \textbf{0.306}$_{\pm0.008}$ & \textbf{0.903}$_{\pm0.005}$ & \textbf{0.907}$_{\pm0.056}$ &  \\
\addlinespace[2pt]
\multirow{2}{*}{$(1,2,1,1)$} & static & 2.709$_{\pm0.130}$ & 0.314$_{\pm0.005}$ & 0.299$_{\pm0.028}$ & 0.889$_{\pm0.010}$ & \textbf{0.908}$_{\pm0.060}$ & \multirow{2}{*}{mixed} \\
 & \oursshort{} & \textbf{2.720}$_{\pm0.081}$ & \textbf{0.318}$_{\pm0.002}$ & \textbf{0.314}$_{\pm0.008}$ & \textbf{0.906}$_{\pm0.001}$ & 0.869$_{\pm0.065}$ &  \\
\addlinespace[2pt]
\multirow{2}{*}{$(2,1,1,1)$} & static & \textbf{2.727}$_{\pm0.038}$ & 0.317$_{\pm0.001}$ & 0.302$_{\pm0.008}$ & 0.892$_{\pm0.001}$ & \textbf{0.900}$_{\pm0.030}$ & \multirow{2}{*}{mixed} \\
 & \oursshort{} & 2.723$_{\pm0.017}$ & \textbf{0.323}$_{\pm0.003}$ & 0.302$_{\pm0.007}$ & \textbf{0.901}$_{\pm0.005}$ & 0.874$_{\pm0.012}$ &  \\
\bottomrule
\end{tabular}
\end{table}

\begin{table}[!htbp]
\centering
\small
\setlength{\tabcolsep}{4pt}
\caption{Training stage-GenEval, per-$\bm{\lambda}$ training rewards on the held-out dataset at each run's final checkpoint, static vs.\ \ourmethod{}. Entries, subscripts, bolding, and the Pareto verdict are as in Tab.~\ref{tab:perlam-ocr}.}
\label{tab:preagg-geneval}
\begin{tabular}{@{}llcccccl@{}}
\toprule
$\bm{\lambda}$ & Method & Aggregate $r$ & \textsc{ClipScore} & \textsc{HPSv2} & \textsc{PickScore} & \textsc{GenEval} & Pareto \\
\midrule
\multirow{2}{*}{$(1,1,1,1)$} & static & 2.390$_{\pm0.003}$ & 0.298$_{\pm0.001}$ & \textbf{0.316}$_{\pm0.005}$ & \textbf{0.914}$_{\pm0.002}$ & 0.861$_{\pm0.002}$ & \multirow{2}{*}{mixed} \\
 & \oursshort{} & \textbf{2.392}$_{\pm0.022}$ & 0.298$_{\pm0.003}$ & 0.308$_{\pm0.007}$ & 0.909$_{\pm0.006}$ & \textbf{0.877}$_{\pm0.014}$ &  \\
\addlinespace[2pt]
\multirow{2}{*}{$(1,1,1,2)$} & static & \textbf{3.300}$_{\pm0.017}$ & 0.299$_{\pm0.001}$ & \textbf{0.309}$_{\pm0.004}$ & \textbf{0.908}$_{\pm0.008}$ & \textbf{0.892}$_{\pm0.007}$ & \multirow{2}{*}{mixed} \\
 & \oursshort{} & 3.290$_{\pm0.045}$ & 0.299$_{\pm0.000}$ & 0.307$_{\pm0.007}$ & 0.906$_{\pm0.005}$ & 0.889$_{\pm0.019}$ &  \\
\addlinespace[2pt]
\multirow{2}{*}{$(1,1,2,1)$} & static & 3.321$_{\pm0.026}$ & \textbf{0.298}$_{\pm0.001}$ & 0.319$_{\pm0.006}$ & \textbf{0.914}$_{\pm0.007}$ & 0.875$_{\pm0.014}$ & \multirow{2}{*}{mixed} \\
 & \oursshort{} & \textbf{3.322}$_{\pm0.014}$ & 0.297$_{\pm0.004}$ & \textbf{0.322}$_{\pm0.001}$ & 0.913$_{\pm0.004}$ & \textbf{0.877}$_{\pm0.009}$ &  \\
\addlinespace[2pt]
\multirow{2}{*}{$(1,2,1,1)$} & static & 2.715$_{\pm0.006}$ & 0.299$_{\pm0.002}$ & 0.313$_{\pm0.003}$ & \textbf{0.915}$_{\pm0.001}$ & \textbf{0.876}$_{\pm0.007}$ & \multirow{2}{*}{mixed} \\
 & \oursshort{} & 2.715$_{\pm0.021}$ & 0.299$_{\pm0.002}$ & \textbf{0.322}$_{\pm0.003}$ & 0.908$_{\pm0.003}$ & 0.865$_{\pm0.017}$ &  \\
\addlinespace[2pt]
\multirow{2}{*}{$(2,1,1,1)$} & static & \textbf{2.677}$_{\pm0.071}$ & \textbf{0.301}$_{\pm0.003}$ & \textbf{0.309}$_{\pm0.015}$ & 0.904$_{\pm0.008}$ & 0.863$_{\pm0.042}$ & \multirow{2}{*}{mixed} \\
 & \oursshort{} & 2.671$_{\pm0.029}$ & 0.300$_{\pm0.002}$ & 0.302$_{\pm0.013}$ & \textbf{0.906}$_{\pm0.004}$ & \textbf{0.864}$_{\pm0.012}$ &  \\
\bottomrule
\end{tabular}
\end{table}

\begin{table}[!htbp]
\centering
\small
\setlength{\tabcolsep}{4pt}
\caption{Training stage-OCR, per-$\bm{\lambda}$ held-out judge scores at each run's final checkpoint, static vs.\ \ourmethod{}. Entries, subscripts, bolding, and the Pareto verdict are as in Tab.~\ref{tab:perlam-ocr}, now over the four judges; the \textsc{HPSv3} column is at two decimals because its scale is an order of magnitude larger.}
\label{tab:perlam-panel-ocr}
\begin{tabular}{@{}llccccl@{}}
\toprule
$\bm{\lambda}$ & Method & \textsc{ImageReward} & \textsc{Aesthetic} & \textsc{HPSv3} & \textsc{UnifiedReward-2} & Pareto \\
\midrule
\multirow{2}{*}{$(1,1,1,1)$} & static & 1.433$_{\pm0.056}$ & 5.493$_{\pm0.096}$ & 12.86$_{\pm0.33}$ & \textbf{3.179}$_{\pm0.071}$ & \multirow{2}{*}{mixed} \\
 & \oursshort{} & \textbf{1.476}$_{\pm0.027}$ & \textbf{5.574}$_{\pm0.063}$ & \textbf{13.16}$_{\pm0.50}$ & 3.154$_{\pm0.030}$ &  \\
\addlinespace[2pt]
\multirow{2}{*}{$(1,1,1,2)$} & static & 1.341$_{\pm0.060}$ & 5.411$_{\pm0.036}$ & 12.06$_{\pm0.71}$ & \textbf{3.146}$_{\pm0.061}$ & \multirow{2}{*}{mixed} \\
 & \oursshort{} & \textbf{1.420}$_{\pm0.096}$ & \textbf{5.542}$_{\pm0.027}$ & \textbf{12.44}$_{\pm0.81}$ & 3.129$_{\pm0.055}$ &  \\
\addlinespace[2pt]
\multirow{2}{*}{$(1,1,2,1)$} & static & 1.438$_{\pm0.042}$ & 5.525$_{\pm0.083}$ & 13.00$_{\pm0.18}$ & 3.159$_{\pm0.008}$ & \multirow{2}{*}{\oursshort{} $\succ$ static} \\
 & \oursshort{} & \textbf{1.499}$_{\pm0.029}$ & \textbf{5.571}$_{\pm0.045}$ & \textbf{13.25}$_{\pm0.61}$ & \textbf{3.190}$_{\pm0.065}$ &  \\
\addlinespace[2pt]
\multirow{2}{*}{$(1,2,1,1)$} & static & 1.434$_{\pm0.066}$ & 5.481$_{\pm0.093}$ & 12.67$_{\pm0.54}$ & 3.129$_{\pm0.071}$ & \multirow{2}{*}{\oursshort{} $\succ$ static} \\
 & \oursshort{} & \textbf{1.517}$_{\pm0.003}$ & \textbf{5.622}$_{\pm0.075}$ & \textbf{13.44}$_{\pm0.39}$ & \textbf{3.181}$_{\pm0.006}$ &  \\
\addlinespace[2pt]
\multirow{2}{*}{$(2,1,1,1)$} & static & 1.461$_{\pm0.032}$ & 5.521$_{\pm0.025}$ & 13.13$_{\pm0.35}$ & 3.158$_{\pm0.028}$ & \multirow{2}{*}{\oursshort{} $\succ$ static} \\
 & \oursshort{} & \textbf{1.499}$_{\pm0.025}$ & \textbf{5.581}$_{\pm0.046}$ & \textbf{13.41}$_{\pm0.40}$ & \textbf{3.226}$_{\pm0.070}$ &  \\
\bottomrule
\end{tabular}
\end{table}

\begin{table}[!htbp]
\centering
\small
\setlength{\tabcolsep}{4pt}
\caption{Training stage-GenEval, per-$\bm{\lambda}$ held-out judge scores at each run's final checkpoint, static vs.\ \ourmethod{}. Entries, subscripts, bolding, and the Pareto verdict are as in Tab.~\ref{tab:perlam-panel-ocr}.}
\label{tab:perlam-panel-geneval}
\begin{tabular}{@{}llccccl@{}}
\toprule
$\bm{\lambda}$ & Method & \textsc{ImageReward} & \textsc{Aesthetic} & \textsc{HPSv3} & \textsc{UnifiedReward-2} & Pareto \\
\midrule
\multirow{2}{*}{$(1,1,1,1)$} & static & \textbf{1.372}$_{\pm0.011}$ & \textbf{6.057}$_{\pm0.020}$ & 8.73$_{\pm0.51}$ & \textbf{2.745}$_{\pm0.025}$ & \multirow{2}{*}{mixed} \\
 & \oursshort{} & 1.358$_{\pm0.046}$ & 6.006$_{\pm0.076}$ & \textbf{9.30}$_{\pm0.12}$ & 2.741$_{\pm0.025}$ &  \\
\addlinespace[2pt]
\multirow{2}{*}{$(1,1,1,2)$} & static & 1.268$_{\pm0.073}$ & 5.882$_{\pm0.128}$ & 8.43$_{\pm1.01}$ & 2.705$_{\pm0.065}$ & \multirow{2}{*}{\oursshort{} $\succ$ static} \\
 & \oursshort{} & \textbf{1.317}$_{\pm0.106}$ & \textbf{5.938}$_{\pm0.129}$ & \textbf{8.61}$_{\pm0.72}$ & \textbf{2.709}$_{\pm0.042}$ &  \\
\addlinespace[2pt]
\multirow{2}{*}{$(1,1,2,1)$} & static & 1.234$_{\pm0.203}$ & 5.913$_{\pm0.228}$ & 8.19$_{\pm1.58}$ & 2.698$_{\pm0.125}$ & \multirow{2}{*}{\oursshort{} $\succ$ static} \\
 & \oursshort{} & \textbf{1.386}$_{\pm0.061}$ & \textbf{6.040}$_{\pm0.090}$ & \textbf{8.70}$_{\pm0.62}$ & \textbf{2.736}$_{\pm0.030}$ &  \\
\addlinespace[2pt]
\multirow{2}{*}{$(1,2,1,1)$} & static & 0.990$_{\pm0.562}$ & 5.699$_{\pm0.380}$ & 7.23$_{\pm2.74}$ & 2.588$_{\pm0.250}$ & \multirow{2}{*}{\oursshort{} $\succ$ static} \\
 & \oursshort{} & \textbf{1.369}$_{\pm0.033}$ & \textbf{6.054}$_{\pm0.037}$ & \textbf{8.46}$_{\pm0.42}$ & \textbf{2.706}$_{\pm0.025}$ &  \\
\addlinespace[2pt]
\multirow{2}{*}{$(2,1,1,1)$} & static & \textbf{1.127}$_{\pm0.190}$ & 5.768$_{\pm0.223}$ & 7.49$_{\pm1.01}$ & 2.618$_{\pm0.096}$ & \multirow{2}{*}{mixed} \\
 & \oursshort{} & 1.014$_{\pm0.540}$ & \textbf{5.800}$_{\pm0.337}$ & \textbf{7.81}$_{\pm2.33}$ & \textbf{2.633}$_{\pm0.182}$ &  \\
\bottomrule
\end{tabular}
\end{table}

\section{Experiment details}
\subsection{Reward models}
\label{app:rewards}
Tab.~\ref{tab:rewards} documents every reward model used anywhere in the paper: its architecture, checkpoint or backbone, and native output range. Two roles recur: \emph{train}, entering the training objective through the weight matrix $\bm{W}$; \emph{held-out}, scoring checkpoints it never influenced. Sec.~\ref{sec:empirical-omega} also estimates gain curves for two of the held-out judges (\textsc{Aesthetic}, \textsc{ImageReward}), to show what the estimator produces on rewards of a different type; they enter no training kernel and no training loss.

The training-stage kernels are built from \textsc{OCR}'s and \textsc{GenEval}'s own gain curves, estimated exactly as in Sec.~\ref{sec:empirical-omega} but on their own prompt sets and their own $T{=}25$ schedule. Fig.~\ref{fig:renyi-omega} also overlays both, linearly resampled $T{=}25$ onto the $T{=}10$ grid of the other five, purely for visual comparison.

\paragraph{Positivity of the training rewards.} The analysis of Sec.~\ref{sec:Dt} assumes $r_i\geq0$, so we state how each training reward is computed. Write $\cos(\bm{u},\bm{v})$ for the cosine similarity of an image embedding $\bm{u}$ and a text embedding $\bm{v}$. \textsc{ClipScore} is $\cos$ between the CLIP ViT-L/14 image and text embeddings. \textsc{HPSv2} is $\cos$ between the $\ell_2$-normalized image and text embeddings of the preference-fine-tuned ViT-H/14, with no logit scale applied. \textsc{PickScore} is the preference head's $\cos$ rescaled by the model's own learned logit scale and a fixed constant, $(e^{s}/26)\cos$ with $e^{s}\approx98.9$, so $\approx3.80\cos$. \textsc{OCR} is $1-\min\{\mathrm{Lev}(\hat s,s),\,|s|\}/|s|$, where $s$ is the target string quoted in the prompt, $\hat s$ is the recognized text after lowercasing and removing whitespace, $\mathrm{Lev}$ is the Levenshtein distance, and the distance is taken as $0$ whenever $s$ occurs as a substring of $\hat s$. \textsc{GenEval} is the fraction of the prompt's object, count, color, and position clauses that the detector confirms.

\textsc{OCR} and \textsc{GenEval} are therefore non-negative by construction, one a distance ratio capped at $1$ and the other a fraction of satisfied clauses. The remaining three are cosine similarities up to a positive scale, so they are sign-indefinite in principle, with ranges $[-1,1]$, $[-1,1]$, and $\approx[-3.80,3.80]$. We clamp the reward at $0$. This clamping has no effect in estimation, because on our prompt sets they are positive throughout and with a wide margin: over every trained checkpoint we evaluate, the smallest \textsc{ClipScore} is $0.26$, the smallest \textsc{HPSv2} is $0.23$, and the smallest \textsc{PickScore} is $0.77$, against a floor of $0$. This reflects a property of CLIP-style encoders rather than a coincidence, since image and text embeddings occupy separate cones and their cosines concentrate in a narrow positive band. We apply a sigmoid to \textsc{ImageReward}, whose raw output is a signed comparison scalar, when estimating its gain curve in Sec.~\ref{sec:empirical-omega}.

\begin{table}[!htbp]
\centering
\small
\caption{Reward models used in the paper. ``Range'' is the reward's output scale; several are unbounded in principle but concentrate in the interval shown in practice.}
\label{tab:rewards}
\begin{tabular}{@{}>{\raggedright\arraybackslash}p{2.5cm}>{\raggedright\arraybackslash}p{6.1cm}>{\raggedright\arraybackslash}p{1.5cm}>{\raggedright\arraybackslash}p{2.4cm}@{}}
\toprule
Reward & Architecture / checkpoint & Range & Role \\
\midrule
\textsc{ClipScore} \citep{hessel2021clipscore} & CLIP ViT-L/14 (\path{openai/clip-vit-large-patch14}) image--text cosine similarity & $\approx[0,1]$ & train \\
\midrule
\textsc{HPSv2} \citep{wu2023hpsv2} & open\_clip ViT-H/14 backbone fine-tuned on human preference pairs (\path{HPS_v2.1_compressed.pt} checkpoint) & $\approx[0,1]$ & train \\
\midrule
\textsc{PickScore} \citep{kirstain2023pick} & \path{yuvalkirstain/PickScore_v1}, a CLIP ViT-H/14 fine-tuned end-to-end on preference pairs, with the \texttt{laion/\allowbreak CLIP-ViT-H-14-\allowbreak laion2B-\allowbreak s32B-\allowbreak b79K} processor & $\approx[0,3.8]$ & train \\
\midrule
\textsc{OCR} & PaddleOCR (English, angle classification off) text extraction& $[0,1]$ & train (training stage-OCR) \\
\midrule
\textsc{GenEval} \citep{ghosh2023geneval} & Mask2Former (Swin-S, COCO) detector + CLIP zero-shot color/counting classifiers & $[0,1]$ & train (training stage-GenEval) \\
\midrule
\textsc{Aesthetic} \citep{schuhmann2022aesthetics} & Linear-probe MLP over frozen CLIP ViT-L/14 features, trained on LAION/AVA aesthetic ratings (\path{sac+logos+ava1-l14-linearMSE} checkpoint) & $\approx[1,10]$ & held-out \\
\midrule
\textsc{ImageReward} \citep{xu2023imagereward} & Official \path{ImageReward-v1.0} checkpoint & unbounded, $z$-scored & held-out \\
\midrule
\textsc{HPSv3} \citep{ma2025hpsv3} & \path{HPSv3RewardInferencer}, a $\sim$7B-parameter vision-language reward model & unbounded, $\approx[-10,16]$ & held-out \\
\midrule
\textsc{UnifiedReward-2} \citep{wang2025unified} & \path{CodeGoat24/UnifiedReward-2.0-qwen3vl-8b}, a Qwen3-VL-8B pointwise judge served through vLLM (greedy decoding, temperature $0$) & $[1,5]$ & held-out \\
\bottomrule
\end{tabular}
\end{table}

\subsection{Training hyperparameters and compute}
\label{app:hparams}

\paragraph{Architecture and optimization.} All runs fine-tune LoRA adapters (rank $32$, scale $64$, Gaussian initialization) inserted into the eight attention-projection modules of the SD3.5-Medium joint-attention transformer block; the transformer backbone itself is frozen. The optimizer is AdamW with learning rate $3\times10^{-4}$, $(\beta_1,\beta_2)=(0.9,0.999)$, $\epsilon=10^{-8}$, weight decay $10^{-4}$, and a decayed learning-rate schedule (\texttt{decay\_type=1}) shared by every stage. The DiffusionNFT coefficient is $\beta=0.1$ throughout, sampling uses $T=25$ steps, and the rollout batch is $24$ images/prompt on each of $8$ GPUs with $6$ gradient-accumulation steps, for an effective batch of $8\times24\times6=1152$. Evaluation images are drawn at seed $42{+}$index, identical across the two methods of a pair. Sinkhorn kernels are built once per $\bm{\lambda}$ before training starts, at $\alpha=2$ for every training reward.

\paragraph{Compute.} All training ran on single-node $8\times$NVIDIA H200 allocations. A warmup-stage run costs $\approx68$ GPU-hours, a training stage-OCR run $\approx37$, and a training stage-GenEval run $\approx59$. The static and \ourmethod{} methods of a pair cost the same: the projection is negligible next to sampling and reward scoring.

\paragraph{Gain curve construction cost.} Each curve uses $N{=}8$ $\pi^+$ samples per prompt, and $K{=}16$ $\pi^{\mathrm{old}}$ rollouts from every intermediate $\vx_t$; every reward of a stage is scored on the same rollouts, so the cost is per curve rather than per reward, and App.~\ref{sec:limitations} reports what reproducing one curve from scratch costs. Ignoring the \textsc{GPT Image 1.5} API time, the $P{\cdot}N{\cdot}T{\cdot}K$ rollouts of one $T{=}10$ curve take $\approx15$ GPU-hours, under half the $\approx37$ of a single training stage-OCR run, and a training-stage curve at $T{=}25$ takes about $6\times$ those $15$ GPU-hours. The $P{\cdot}N$ queries to \textsc{GPT Image 1.5} cost around \$35 in API fees. Crucially, this is a one-time cost per stage: each curve is estimated once at the untrained policy and then reused, unmodified, so its amortized cost per training run is well below the raw totals.

\subsection{MMRBv2 judge: protocol, position debiasing, and prompt}
\label{app:MMRBv2-details}

\paragraph{Protocol.} The pairwise judgments of Sec.~\ref{sec:MMRBv2} use the image-generation judge from Meta's MMRBv2~\citep{hu2026multimodal} evaluation suite, served by \texttt{gemini-3.5-flash} at temperature $0$. For each of the five matched $\bm{\lambda}$ budgets, we pair the static method and the \ourmethod{} method image-by-image on the fixed $1{,}000$-prompt subset of the held-out OCR dataset of Tab.~\ref{tab:panels}. Each API call sends the rubric prompt below, the original text prompt, and the two images labeled \texttt{[RESPONSE A:]} and \texttt{[RESPONSE B:]}. The judge returns a structured JSON verdict: step-by-step reasoning over seven criteria (faithfulness to prompt, text rendering, input faithfulness, image consistency, text--image alignment, text quality, overall quality), a scalar score $s\in\{1,\dots,6\}$ where $s\geq 4$ favors the image shown as A, a discrete \texttt{better\_response} $\in\{$A, B$\}$, and a self-reported confidence. We take the winner from \texttt{better\_response}; calls that still fail after $8$ retries drop the pair.

\paragraph{Position debiasing.} LLM judges exhibit position bias: the response presented first is systematically favored. Following MMRBv2's forward+reverse protocol, every pair is therefore judged \emph{twice}: once with the static image shown as A and \ourmethod{} as B (forward), and once with the images swapped (reverse). Each ordering credits one win to the method the judge picks, after mapping the shown-side label back to the method's identity. The win rate over $n$ prompts is the fraction of the $2n$ orderings won,
\begin{equation*}
\mathrm{win}(\text{\oursshort{}}) \;=\; \frac{1}{2n}\sum_{j=1}^{n} \#\{\text{orderings of pair } j \text{ won by \oursshort{}}\}.
\end{equation*}

\paragraph{Faithfulness and aesthetic sub-scores.} The judge is not asked for numeric per-criterion scores; the \emph{Faithfulness} and \emph{Aesthetics} rows of Tab.~\ref{tab:MMRBv2-ocr} are extracted from its per-criterion reasoning fields \texttt{faithfulness\_to\_prompt} and \texttt{overall\_quality}, respectively. For each ordering, the criterion's reasoning text is mapped to a winner in $\{$A, B, undecided$\}$ by matching declarative phrasings like ``Response A is better\ldots'' and ``the edge goes to B\ldots'', with a text-only LLM tie-break when the patterns are ambiguous. The sub-score is the fraction of \emph{decided} orderings won by \ourmethod{}; orderings whose criterion text is undecided or marked not-applicable are dropped from that criterion's denominator.

\paragraph{Judge prompt.} The rubric prompt, verbatim from MMRBv2's image-generation judge:

{\scriptsize
\begin{verbatim}
You are an expert in multimodal quality analysis and generative AI evaluation. Your
role is to act as an objective judge for comparing two AI-generated responses to the
same prompt. You will evaluate which response is better based on a comprehensive
rubric.

**Important Guidelines:**
- Be completely impartial and avoid any position biases
- Ensure that the order in which the responses were presented does not influence
  your decision
- Do not allow the length of the responses to influence your evaluation
- Do not favor certain model names or types
- Be as objective as possible in your assessment
- Consider factors such as helpfulness, relevance, accuracy, depth, creativity, and
  level of detail

**Understanding the Content Structure:**
- **[ORIGINAL PROMPT TO MODEL:]**: This is the instruction given to both AI models
- **[INPUT IMAGE FROM PROMPT:]**: This is the source image provided to both models
  (if any)
- **[RESPONSE A:]**: The first model's generated response (text and/or images)
- **[RESPONSE B:]**: The second model's generated response (text and/or images)

Your evaluation must be based on a fine-grained rubric that covers the following
criteria. For each criterion, you must provide detailed step-by-step reasoning
comparing both responses. You will use a 1-6 scoring scale.

**Evaluation Criteria:**
1. **faithfulness_to_prompt:** Which response better adheres to the composition,
   objects, attributes, and spatial relationships described in the text prompt?

2. **text_rendering:** If either response contains rendered text, which one has
   better text quality (spelling, legibility, integration)? If no text is rendered,
   state "Not Applicable."

3. **input_faithfulness:** If an input image is provided, which response better
   respects and incorporates the key elements and style of that source image? If no
   input image is provided, state "Not Applicable."

4. **image_consistency:** If multiple images are generated, which response has
   better visual consistency between images (character appearance, scene details)?
   If no multiple images are provided, state "Not Applicable."

5. **text_image_alignment:** Which response has better alignment between text
   descriptions and visual content?

6. **text_quality:** If text was generated, which response has better linguistic
   quality (correctness, coherence, grammar, tone)?

7. **overall_quality:** Which response has better general technical and aesthetic
   quality, realism, coherence, and fewer visual artifacts or distortions?

**Scoring Rubric:**
- Score 6 (A is significantly better): Response A is significantly superior across
  most criteria
- Score 5 (A is marginally better): Response A is noticeably better across several
  criteria
- Score 4 (Unsure or A is negligibly better): Response A is slightly better or
  roughly equivalent
- Score 3 (Unsure or B is negligibly better): Response B is slightly better or
  roughly equivalent
- Score 2 (B is marginally better): Response B is noticeably better across several
  criteria
- Score 1 (B is significantly better): Response B is significantly superior across
  most criteria

**Confidence Assessment:**
After your evaluation, assess your confidence in this judgment on a scale of 0.0
to 1.0:

**CRITICAL**: Be EXTREMELY conservative with confidence scores. Most comparisons
should be in the 0.2-0.5 range.

- **Very High Confidence (0.8-1.0)**: ONLY for absolutely obvious cases where one
  response is dramatically better across ALL criteria with zero ambiguity. Use this
  extremely rarely (less than 10% of cases).
- **High Confidence (0.6-0.7)**: Clear differences but some uncertainty remains.
  Use sparingly (less than 20% of cases).
- **Medium Confidence (0.4-0.5)**: Noticeable differences but significant
  uncertainty. This should be your DEFAULT range.
- **Low Confidence (0.2-0.3)**: Very close comparison, difficult to distinguish.
  Responses are roughly equivalent or have conflicting strengths.
- **Very Low Confidence (0.0-0.1)**: Essentially indistinguishable responses or
  major conflicting strengths.

**IMPORTANT GUIDELINES**:
- DEFAULT to 0.3-0.5 range for most comparisons
- Only use 0.6+ when you are absolutely certain
- Consider: Could reasonable people disagree on this comparison?
- Consider: Are there any strengths in the "worse" response?
- Consider: How obvious would this be to a human evaluator?
- Remember: Quality assessment is inherently subjective

After your reasoning, you will provide a final numerical score, indicate which
response is better, and assess your confidence. You must always output your
response in the following structured JSON format:

{
    "reasoning": {
        "faithfulness_to_prompt": "YOUR REASONING HERE",
        "text_rendering": "YOUR REASONING HERE",
        "input_faithfulness": "YOUR REASONING HERE",
        "image_consistency": "YOUR REASONING HERE",
        "text_image_alignment": "YOUR REASONING HERE",
        "text_quality": "YOUR REASONING HERE",
        "overall_quality": "YOUR REASONING HERE",
        "comparison_summary": "YOUR OVERALL COMPARISON SUMMARY HERE"
    },
    "score": <int 1-6>,
    "better_response": "A" or "B",
    "confidence": <float 0.0-1.0>,
    "confidence_rationale": "YOUR CONFIDENCE ASSESSMENT REASONING HERE"
}
\end{verbatim}
}
\section{Limitations}
\label{sec:limitations}

\paragraph{Approximate positive-policy sampling.}
In practice, we use \textsc{GPT Image 1.5}, an external expert generator, as a surrogate for $\pi^+$ for every reward. Its distribution need not equal any reward-specific $\pi_{i}^+$; without an importance correction, the resulting gain curves are therefore biased estimates of the theoretical $\Delta D_{i,t}$. The empirical curves and ablations show that this proxy is useful, and swapping in an independent surrogate generator leaves the curve essentially unchanged (App.~\ref{app:surrogate-sensitivity}), but neither removes this bias. Future work could construct reward-specific positive policies, estimate the required density correction, or update the curves on-policy as training progresses; App.~\ref{app:curve-drift} re-estimates the curve at a trained checkpoint.

\paragraph{Reproducibility and API cost.} Depending on \textsc{GPT Image 1.5} also means depending on a proprietary, versioned, and priced API. Even though it is a one-time, per-stage expense amortized over many downstream training runs, an independent group reproducing a single curve from scratch pays around \$35 in API costs. Releasing the estimated gain-curve weight files themselves alongside the code would let downstream use of \ourmethod{} skip re-querying \textsc{GPT Image 1.5} entirely.

\section{Qualitative examples}
\label{app:qualitative-gallery}

The two galleries below (Figs.~\ref{fig:appendix-gallery-ocr} and~\ref{fig:appendix-gallery-geneval}) show paired generations from the static baseline and \ourmethod{} for the \textsc{OCR} and \textsc{GenEval} settings. We state explicitly how these examples were selected so that they are not mistaken for random or typical samples. Every one of the $20$ pairs shown ($2$ per $\bm{\lambda}$ budget, with $5$ budgets per setting) is a decisive, judge-verified win for \ourmethod{}, drawn from the pool of ``high-confidence'' pairs identified by MMRBv2 pairwise-judge runs following the protocol in Sec.~\ref{sec:MMRBv2} and App.~\ref{app:MMRBv2-details}, applied separately to each setting's own held-out prompts. For the \textsc{GenEval} setting, this selection uses an analogous MMRBv2 run whose aggregate results are not reported in the main text, since Sec.~\ref{sec:MMRBv2} reports only the \textsc{OCR} setting. We call a win \emph{decisive} when \ourmethod{} wins under both presentation orderings with extreme scores and neither the faithfulness nor the aesthetic sub-criterion contradicts the overall verdict. Within each budget's qualifying pool, the displayed pair was then selected by a human reviewer rather than sampled uniformly at random. These galleries should therefore be interpreted as illustrating what \ourmethod{}'s most decisive wins look like, rather than its average-case behavior. Representative aggregate results are reported in Tab.~\ref{tab:MMRBv2-ocr} for the \textsc{OCR} setting and in Tabs.~\ref{tab:trainreward} and~\ref{tab:panels}, which pool all three seeds, for the \textsc{GenEval} setting.

\begin{figure*}[p]
\centering
\renewcommand{\arraystretch}{1.0}
\setlength{\tabcolsep}{2pt}
{\scriptsize
\begin{tabular}{c c c c c}
\textbf{static} & \textbf{\ourmethod{}} & & \textbf{static} & \textbf{\ourmethod{}} \\
\includegraphics[width=0.185\linewidth]{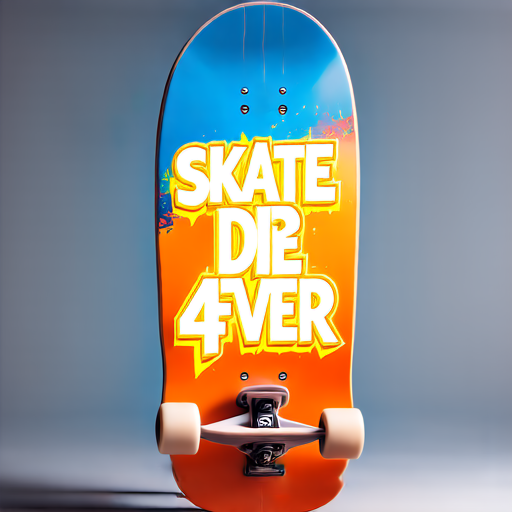} & \includegraphics[width=0.185\linewidth]{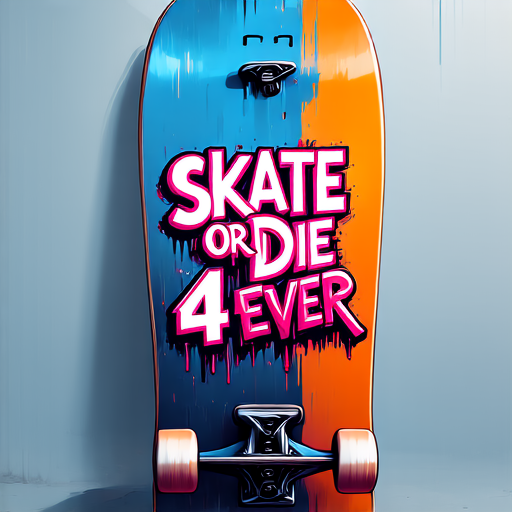} & & \includegraphics[width=0.185\linewidth]{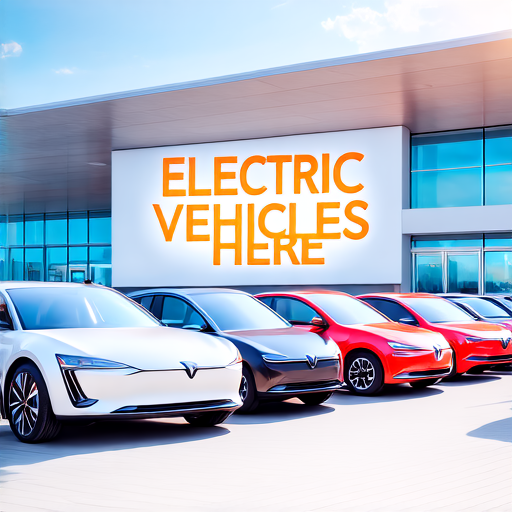} & \includegraphics[width=0.185\linewidth]{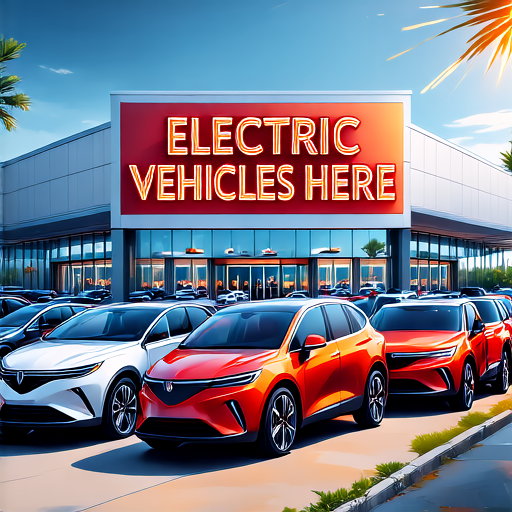} \\
\multicolumn{2}{p{0.42\linewidth}}{\centering \tiny \emph{A vibrant skateboard deck featuring the bold graphic "SKATE OR DIE 4EVER" in dynamic, graffiti-style lettering, set against a gradient background that shifts from deep blue at the nose to bright orange at the tail, with subtle scratch marks and wear to give it a well-used, authentic look.}\\[0pt]\scriptsize $\bm{\lambda}=(1,1,1,1)$} & & \multicolumn{2}{p{0.42\linewidth}}{\centering \tiny \emph{A vibrant car dealership banner reads "Electric Vehicles Here" in bold letters, showcasing a variety of sleek electric cars parked neatly in front of a modern showroom with large glass windows reflecting the sunny sky.}\\[0pt]\scriptsize $\bm{\lambda}=(1,1,1,1)$} \\[1pt]
\includegraphics[width=0.185\linewidth]{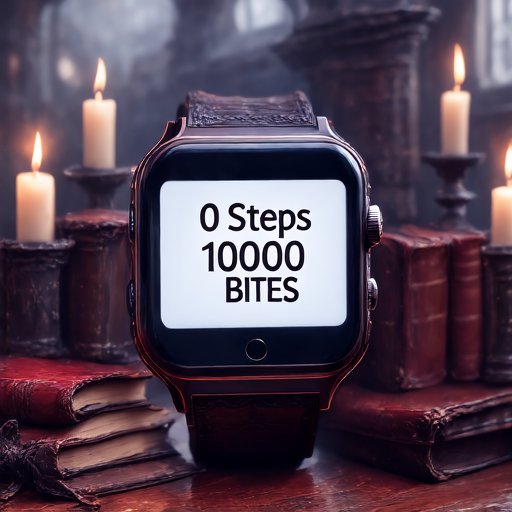} & \includegraphics[width=0.185\linewidth]{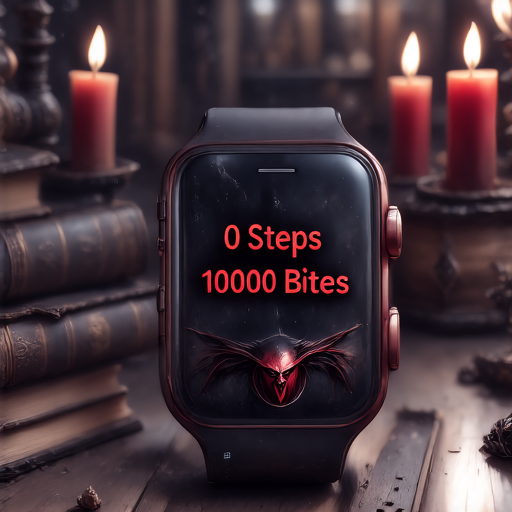} & & \includegraphics[width=0.185\linewidth]{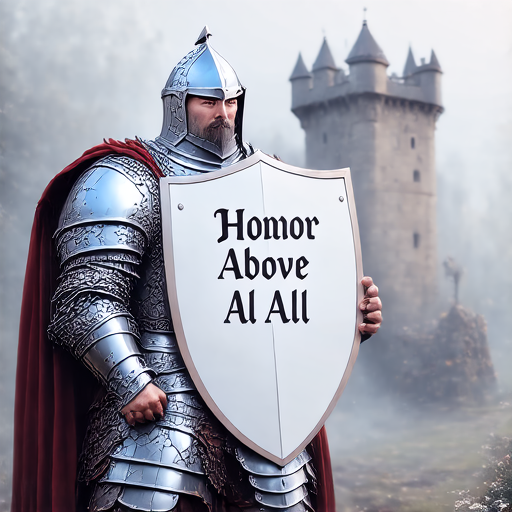} & \includegraphics[width=0.185\linewidth]{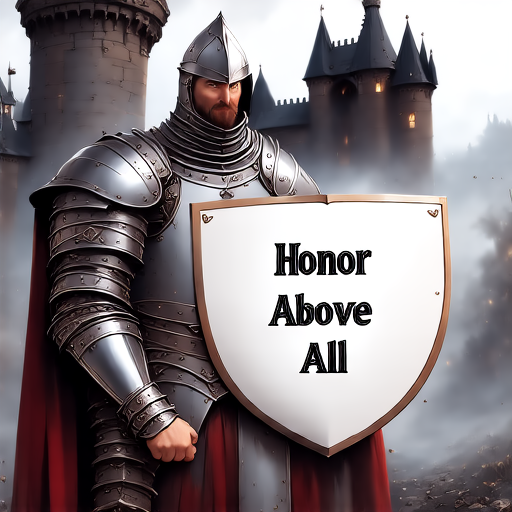} \\
\multicolumn{2}{p{0.42\linewidth}}{\centering \tiny \emph{A sleek, modern vampire fitness tracker displays "0 Steps 10000 Bites" on its screen, resting on a dark, gothic desk surrounded by antique books and candles. The scene is bathed in a dim, eerie light, highlighting the tracker's eerie glow.}\\[0pt]\scriptsize $\bm{\lambda}=(1,1,1,2)$} & & \multicolumn{2}{p{0.42\linewidth}}{\centering \tiny \emph{A medieval knight stands proudly, holding a shield emblazoned with the motto "Honor Above All" in intricate Old English font, set against a backdrop of a misty, ancient castle.}\\[0pt]\scriptsize $\bm{\lambda}=(1,1,1,2)$} \\[1pt]
\includegraphics[width=0.185\linewidth]{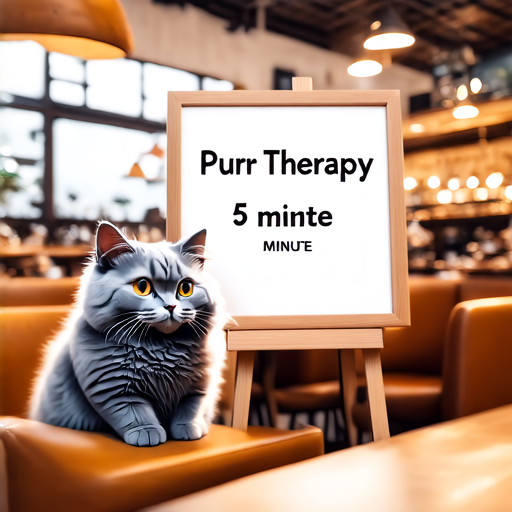} & \includegraphics[width=0.185\linewidth]{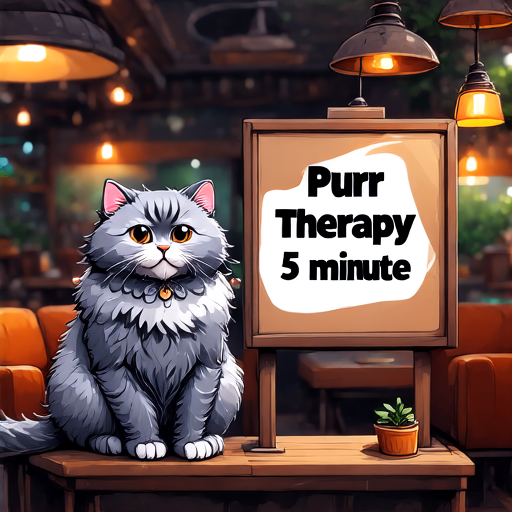} & & \includegraphics[width=0.185\linewidth]{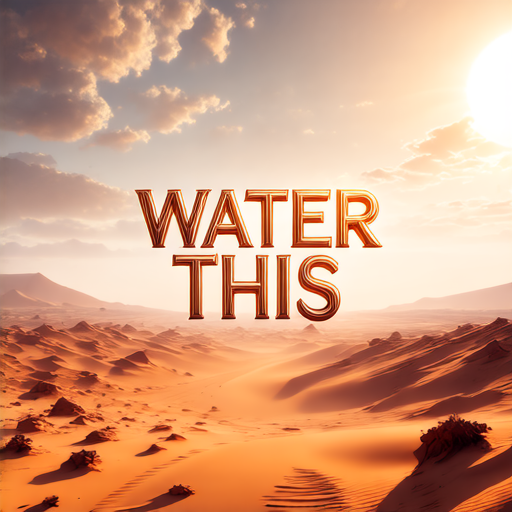} & \includegraphics[width=0.185\linewidth]{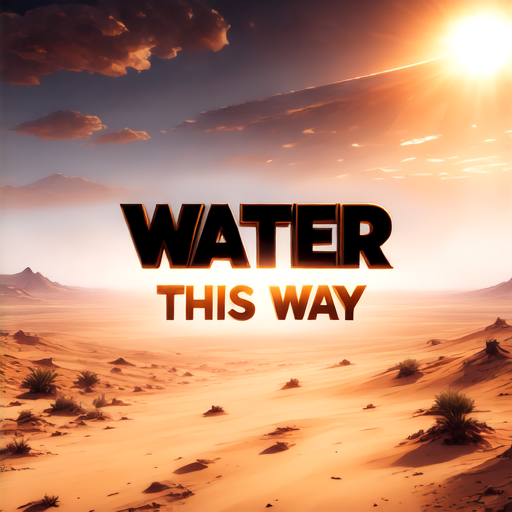} \\
\multicolumn{2}{p{0.42\linewidth}}{\centering \tiny \emph{In a cozy cat cafe, a menu board displays "Purr Therapy 5 minute" among other offerings. A fluffy gray cat sits nearby, looking relaxed and ready to offer its soothing presence to patrons. The scene is warm and inviting, with soft lighting and comfortable seating.}\\[0pt]\scriptsize $\bm{\lambda}=(1,1,2,1)$} & & \multicolumn{2}{p{0.42\linewidth}}{\centering \tiny \emph{A vast desert landscape under a scorching sun, where a mirage forms the shimmering letters "Water This Way" on the distant horizon, creating an illusion of hope in an otherwise barren and arid environment.}\\[0pt]\scriptsize $\bm{\lambda}=(1,1,2,1)$} \\[1pt]
\includegraphics[width=0.185\linewidth]{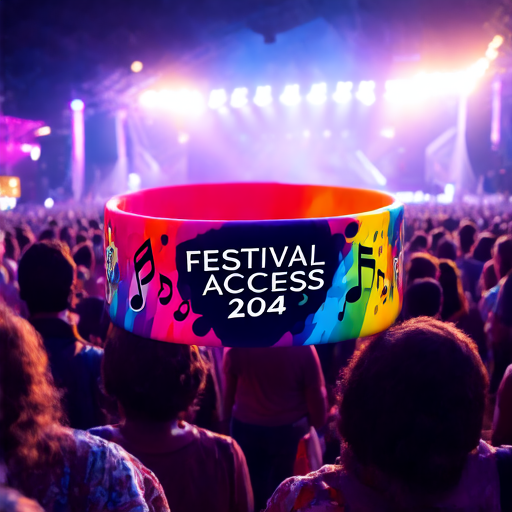} & \includegraphics[width=0.185\linewidth]{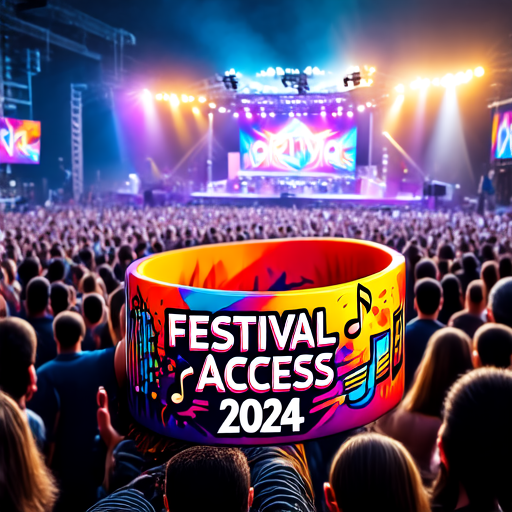} & & \includegraphics[width=0.185\linewidth]{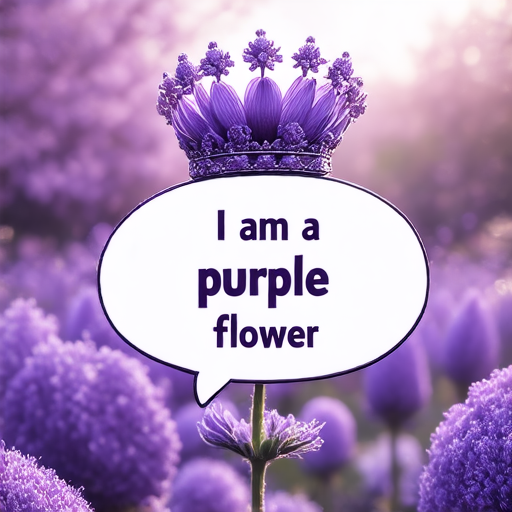} & \includegraphics[width=0.185\linewidth]{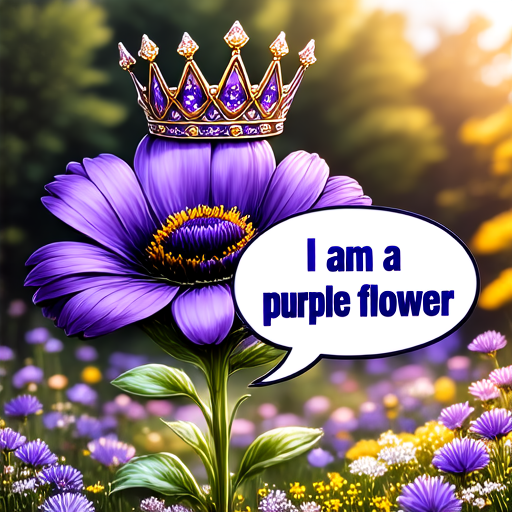} \\
\multicolumn{2}{p{0.42\linewidth}}{\centering \tiny \emph{A vibrant music festival wristband with "Festival Access 2024" prominently displayed, featuring a colorful design with musical notes and festival logos, set against a backdrop of a bustling crowd and stage lights.}\\[0pt]\scriptsize $\bm{\lambda}=(1,2,1,1)$} & & \multicolumn{2}{p{0.42\linewidth}}{\centering \tiny \emph{A purple flower with a delicate crown on its head, featuring a speech bubble that says "I am a purple flower", set against a serene garden backdrop.}\\[0pt]\scriptsize $\bm{\lambda}=(1,2,1,1)$} \\[1pt]
\includegraphics[width=0.185\linewidth]{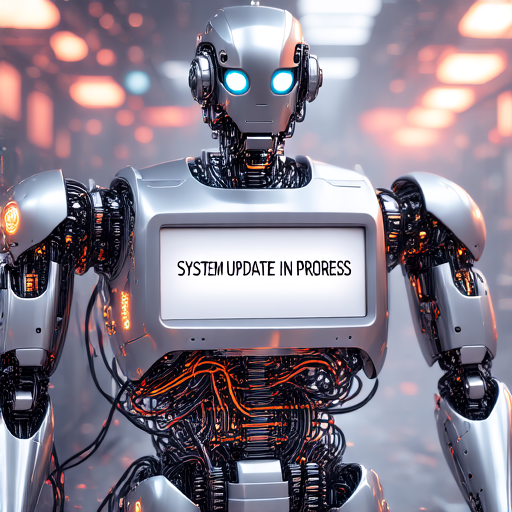} & \includegraphics[width=0.185\linewidth]{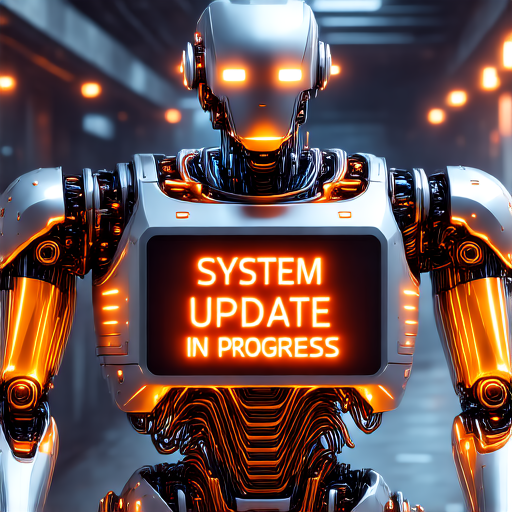} & & \includegraphics[width=0.185\linewidth]{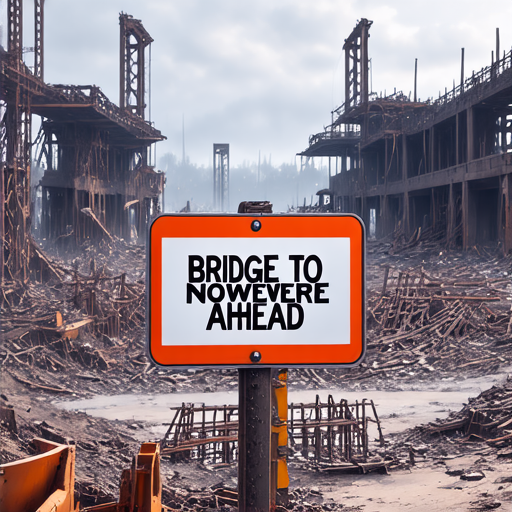} & \includegraphics[width=0.185\linewidth]{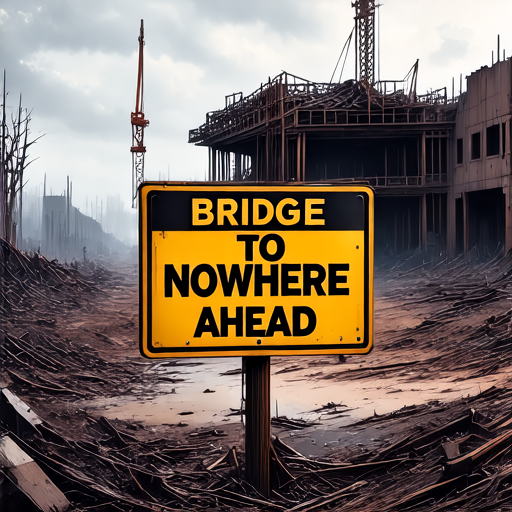} \\
\multicolumn{2}{p{0.42\linewidth}}{\centering \tiny \emph{A close-up of a robot's metallic chest, with a digital display prominently showing "System Update In Progress", surrounded by blinking lights and subtle wiring, set against a dimly lit, futuristic background.}\\[0pt]\scriptsize $\bm{\lambda}=(2,1,1,1)$} & & \multicolumn{2}{p{0.42\linewidth}}{\centering \tiny \emph{A realistic construction site with a warning sign that reads "Bridge to Nowhere Ahead", surrounded by a desolate landscape and half-built structures, emphasizing the abandoned and eerie atmosphere.}\\[0pt]\scriptsize $\bm{\lambda}=(2,1,1,1)$} \\[1pt]
\end{tabular}
}
\caption{Qualitative examples from the \textsc{OCR} setting (training stage-OCR family, held-out $1000$-prompt OCR dataset of Tab.~\ref{tab:panels}): two prompts per $\bm{\lambda}$ budget, static baseline (left of each pair) vs.\ \ourmethod{} (right of each pair). All $10$ pairs are a curated selection of decisive, judge-verified \ourmethod{} wins, not a random or representative sample; see App.~\ref{app:qualitative-gallery} for the selection methodology and Tab.~\ref{tab:MMRBv2-ocr} for this setting's aggregate win rates.}
\label{fig:appendix-gallery-ocr}
\end{figure*}

\begin{figure*}[p]
\centering
\renewcommand{\arraystretch}{1.0}
\setlength{\tabcolsep}{2pt}
{\scriptsize
\begin{tabular}{c c c c c}
\textbf{static} & \textbf{\ourmethod{}} & & \textbf{static} & \textbf{\ourmethod{}} \\
\includegraphics[width=0.185\linewidth]{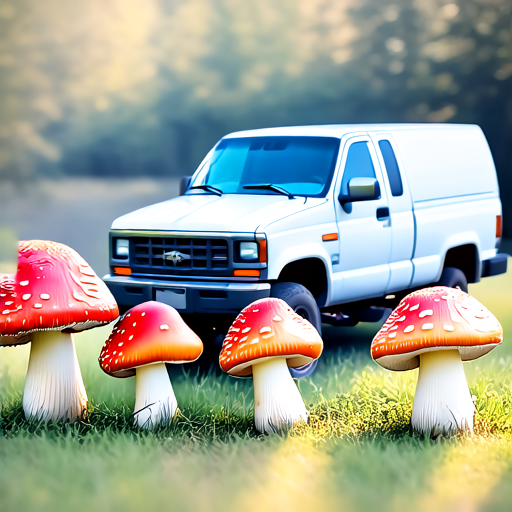} & \includegraphics[width=0.185\linewidth]{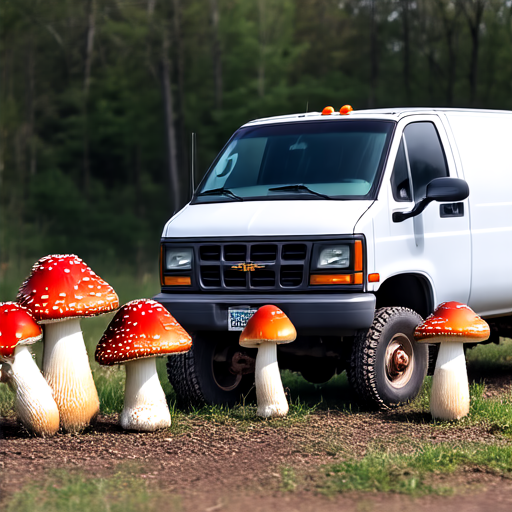} & & \includegraphics[width=0.185\linewidth]{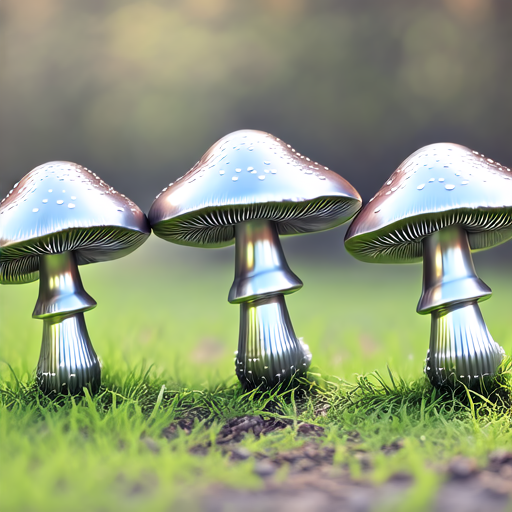} & \includegraphics[width=0.185\linewidth]{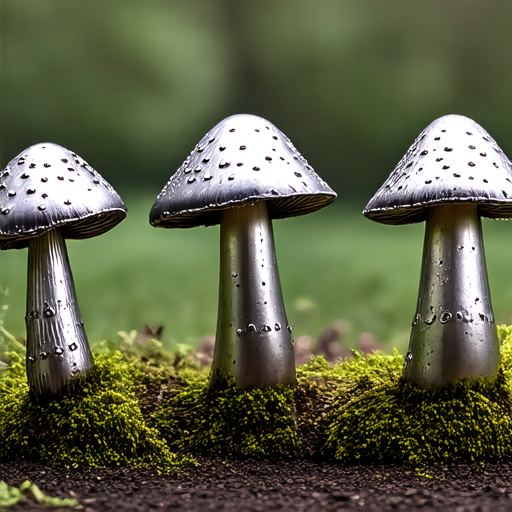} \\
\multicolumn{2}{p{0.42\linewidth}}{\centering \tiny \emph{five mushrooms and a white truck}\\[0pt]\scriptsize $\bm{\lambda}=(1,1,1,1)$} & & \multicolumn{2}{p{0.42\linewidth}}{\centering \tiny \emph{three metal mushrooms}\\[0pt]\scriptsize $\bm{\lambda}=(1,1,1,1)$} \\[1pt]
\includegraphics[width=0.185\linewidth]{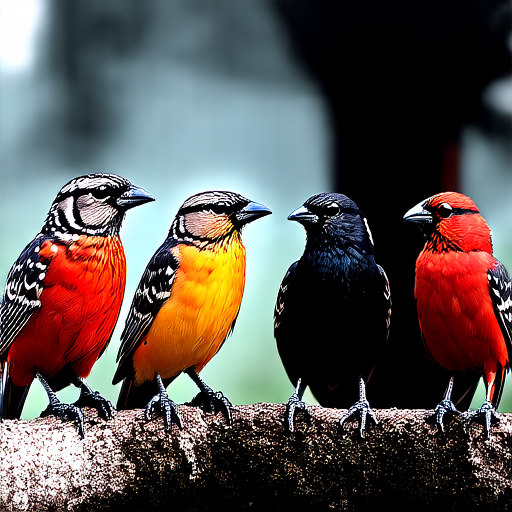} & \includegraphics[width=0.185\linewidth]{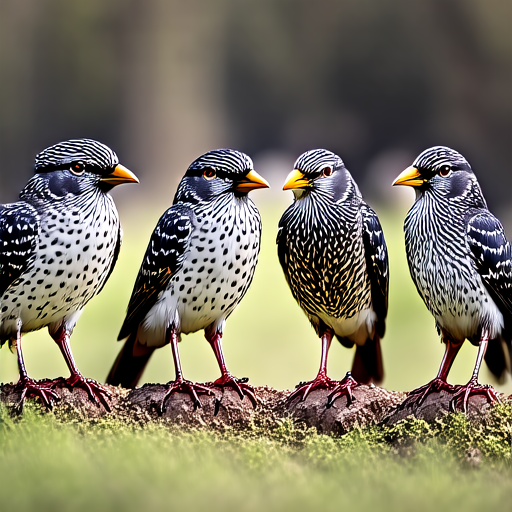} & & \includegraphics[width=0.185\linewidth]{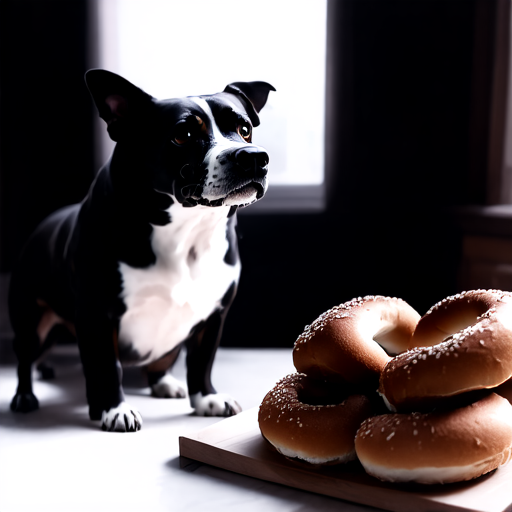} & \includegraphics[width=0.185\linewidth]{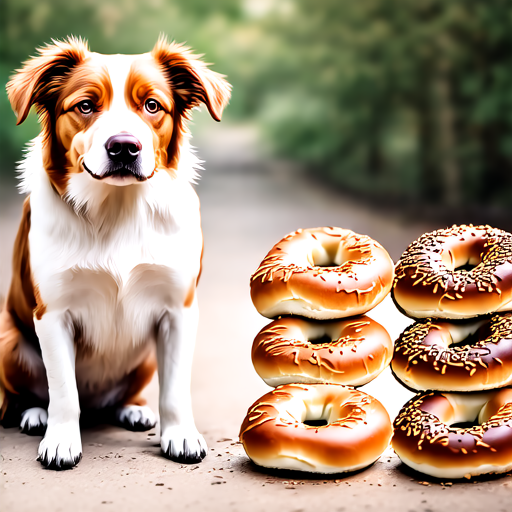} \\
\multicolumn{2}{p{0.42\linewidth}}{\centering \tiny \emph{four spotted birds}\\[0pt]\scriptsize $\bm{\lambda}=(1,1,1,2)$} & & \multicolumn{2}{p{0.42\linewidth}}{\centering \tiny \emph{a dog and six bagels}\\[0pt]\scriptsize $\bm{\lambda}=(1,1,1,2)$} \\[1pt]
\includegraphics[width=0.185\linewidth]{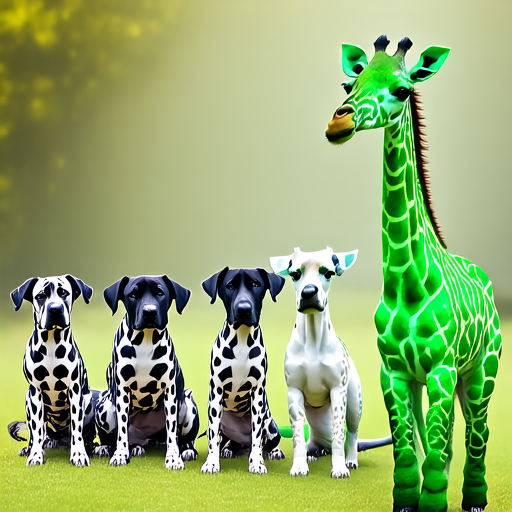} & \includegraphics[width=0.185\linewidth]{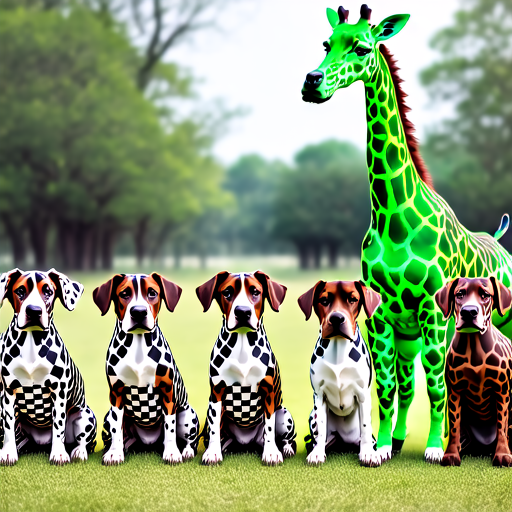} & & \includegraphics[width=0.185\linewidth]{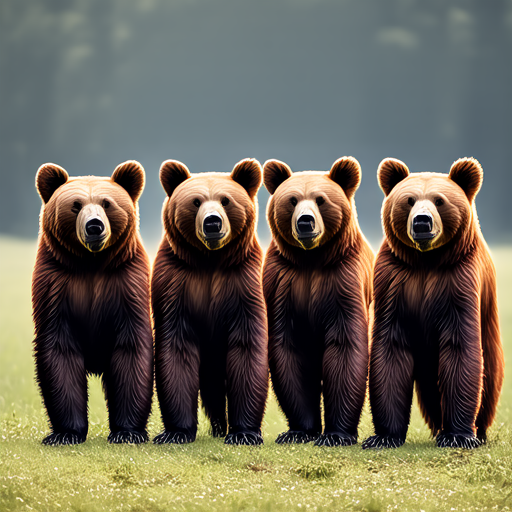} & \includegraphics[width=0.185\linewidth]{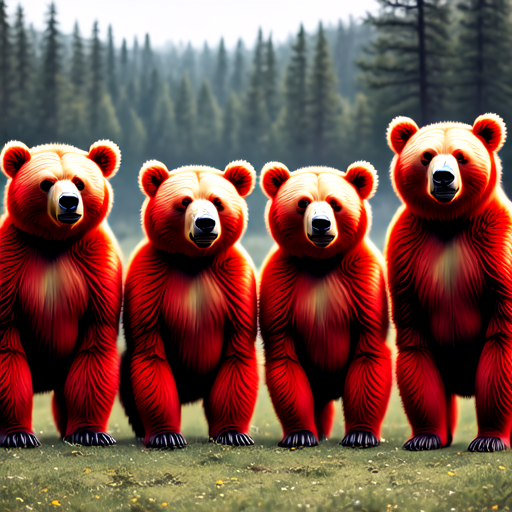} \\
\multicolumn{2}{p{0.42\linewidth}}{\centering \tiny \emph{five checkered dogs and a green giraffe}\\[0pt]\scriptsize $\bm{\lambda}=(1,1,2,1)$} & & \multicolumn{2}{p{0.42\linewidth}}{\centering \tiny \emph{four red bears}\\[0pt]\scriptsize $\bm{\lambda}=(1,1,2,1)$} \\[1pt]
\includegraphics[width=0.185\linewidth]{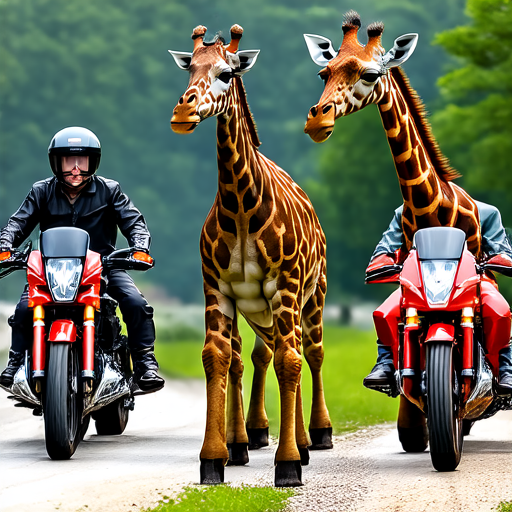} & \includegraphics[width=0.185\linewidth]{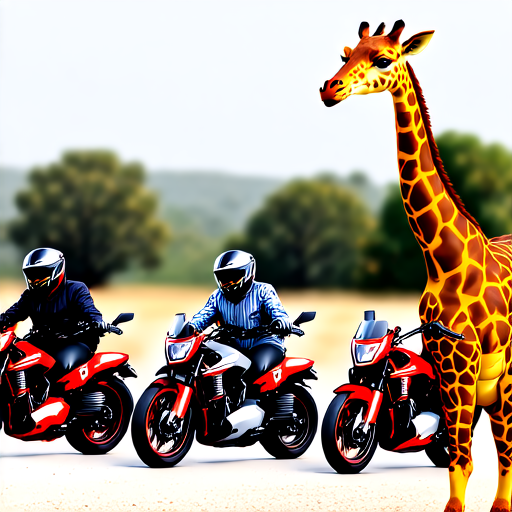} & & \includegraphics[width=0.185\linewidth]{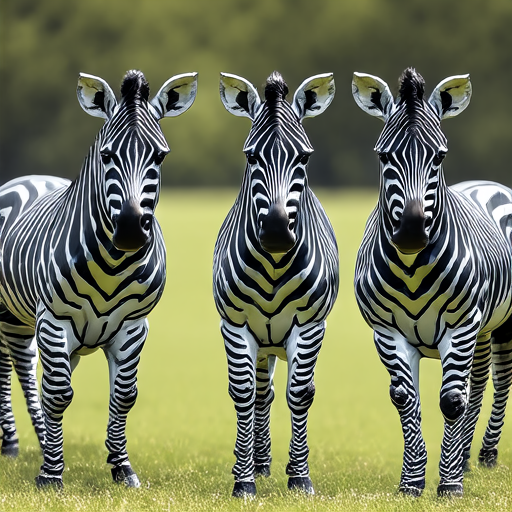} & \includegraphics[width=0.185\linewidth]{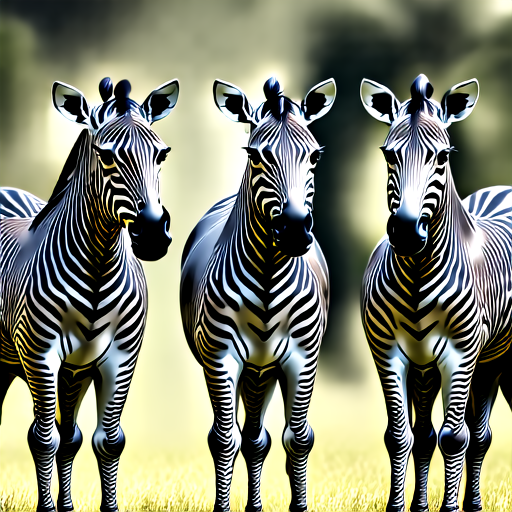} \\
\multicolumn{2}{p{0.42\linewidth}}{\centering \tiny \emph{three motorcycles and a brown giraffe}\\[0pt]\scriptsize $\bm{\lambda}=(1,2,1,1)$} & & \multicolumn{2}{p{0.42\linewidth}}{\centering \tiny \emph{three metal zebras}\\[0pt]\scriptsize $\bm{\lambda}=(1,2,1,1)$} \\[1pt]
\includegraphics[width=0.185\linewidth]{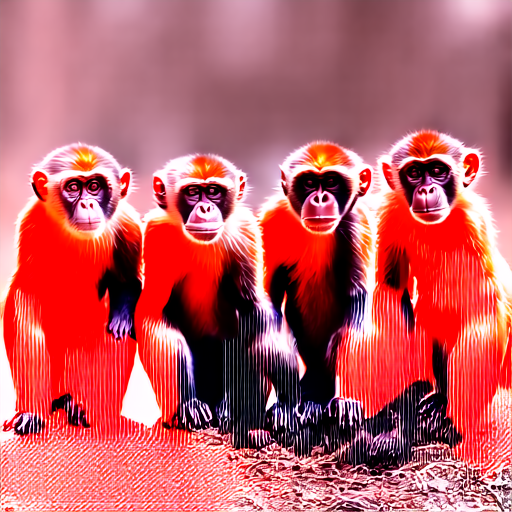} & \includegraphics[width=0.185\linewidth]{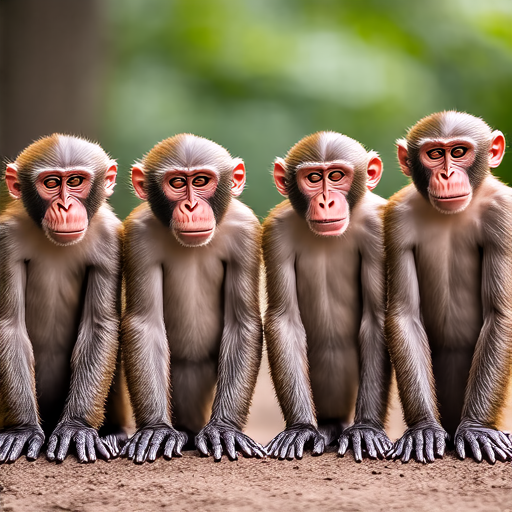} & & \includegraphics[width=0.185\linewidth]{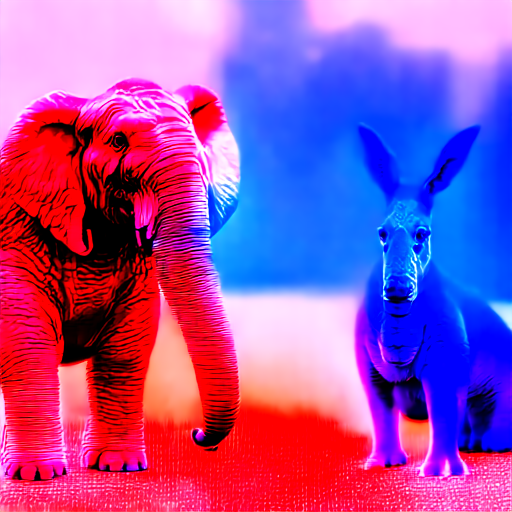} & \includegraphics[width=0.185\linewidth]{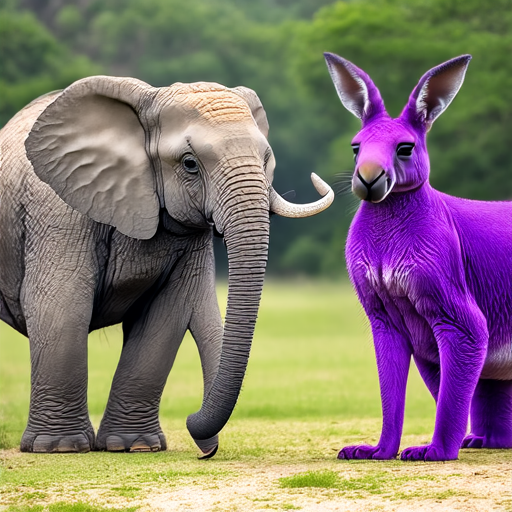} \\
\multicolumn{2}{p{0.42\linewidth}}{\centering \tiny \emph{four brown monkeys}\\[0pt]\scriptsize $\bm{\lambda}=(2,1,1,1)$} & & \multicolumn{2}{p{0.42\linewidth}}{\centering \tiny \emph{a elephant and a purple kangaroo}\\[0pt]\scriptsize $\bm{\lambda}=(2,1,1,1)$} \\[1pt]
\end{tabular}
}
\caption{Qualitative examples from the \textsc{GenEval} setting (training stage-GenEval family, held-out \textsc{GenEval} prompts): two prompts per $\bm{\lambda}$ budget, static baseline (left of each pair) vs.\ \ourmethod{} (right of each pair). All $10$ pairs are a curated selection of decisive, judge-verified \ourmethod{} wins, not a random or representative sample; see App.~\ref{app:qualitative-gallery} for the selection methodology and Tabs.~\ref{tab:trainreward} and~\ref{tab:panels} for this setting's aggregate results; Sec.~\ref{sec:MMRBv2} reports the pairwise judge for the \textsc{OCR} setting only.}
\label{fig:appendix-gallery-geneval}
\end{figure*}

\end{document}